\documentclass[journal]{IEEEtran}
\usepackage{amsmath,amsfonts}
\usepackage{array}
\usepackage[caption=false,font=normalsize,labelfont=sf,textfont=sf]{subfig}
\usepackage{textcomp}
\usepackage{stfloats}
\usepackage{url}
\usepackage{verbatim}
\usepackage{graphicx}
\usepackage{cite}
\usepackage{xcolor}
\usepackage{multirow}
\usepackage{multicol}
\usepackage{pifont}
\usepackage{threeparttable}
\usepackage{hyperref}
\usepackage{xurl}
\usepackage{amssymb}
\usepackage{amsthm}
\usepackage{enumitem}
\usepackage{orcidlink}
\usepackage{pdfpages}
\usepackage[ruled,vlined,linesnumbered]{algorithm2e}

\SetCommentSty{mycommentsty}
\setlist[itemize]{leftmargin=1.2em}

\SetKwInput{KwIn}{Input}
\SetKwInput{KwOut}{Output}
\SetKwFor{ForParEach}{parallel for each}{do}{end}
\newcommand{\StepTcp}[1]{%
  \vspace{0.35em}%
  \tcp{\scalebox{0.95}[1]{#1}}%
}

\newtheorem{lemma}{Lemma}
\newtheorem{corollary}{Corollary}

\newcommand{\xmark}{\ding{55}}%
\newcommand{\tnotea}{\textsuperscript{\textit{a}}}
\newcommand{\tnoteb}{\textsuperscript{\textit{b}}}
\newcommand{\tnotec}{\textsuperscript{\textit{c}}}
\newcommand{\tnoted}{\textsuperscript{\textit{d}}}

\newcommand{\orcidnobox}[1]{{\hypersetup{hidelinks}\orcidlink{#1}}}

\begin{document}

\title{GPU-Accelerated Computation of Persistent Homology for Topological Analysis of Image Data}
\author{
    Fan~Wang\,\orcidnobox{0000-0002-5402-5065},
    Hubert~Wagner\,\orcidnobox{0009-0009-9111-8429},
    Rezaul~Chowdhury\,\orcidnobox{0000-0002-7022-5278},
    and~Chao~Chen\,\orcidnobox{0000-0003-1703-6483}
    \thanks{\copyright~2026 IEEE. Personal use of this material is permitted. Permission from IEEE must be obtained for all other uses, in any current or future media, including reprinting/republishing this material for advertising or promotional purposes, creating new collective works, for resale or redistribution to servers or lists, or reuse of any copyrighted component of this work in other works. Accepted for publication in IEEE Transactions on Pattern Analysis and Machine Intelligence, DOI 10.1109/TPAMI.2026.3741473.}%
    \thanks{Manuscript received January 24, 2026; revised July 19, 2026 and September 11, 2026; accepted October 4, 2026. \textit{(Corresponding authors: Fan Wang and Chao Chen.)}}%
    \thanks{F. Wang and R. Chowdhury are with the Department of Computer Science, and C. Chen is with the Department of Biomedical Informatics, Stony Brook University, Stony Brook, NY 11794 USA (e-mail: \nolinkurl{fanwang1@cs.stonybrook.edu}; \nolinkurl{rezaul@cs.stonybrook.edu}; \nolinkurl{chao.chen.1@stonybrook.edu}).}%
    \thanks{H. Wagner is with the Department of Mathematics, University of Florida, Gainesville, FL 32611 USA (e-mail: \nolinkurl{hwagner@ufl.edu}).}%
}

\markboth{IEEE TRANSACTIONS ON PATTERN ANALYSIS AND MACHINE INTELLIGENCE}%
{Wang \MakeLowercase{\textit{et al.}}: GPU-Accelerated Computation of Persistent Homology for Image Data}


\maketitle

\begin{abstract}
In recent years, persistent homology has seen rapid adoption in deep learning, yet its computation remains a major bottleneck in network training. This paper introduces TopoGPU, a GPU streaming pipeline that computes persistence diagrams of cubical complexes induced by 2D and 3D images. TopoGPU streams the input image chunk by chunk, processing each chunk with massively parallel GPU kernels on a grid of GPU blocks; the resulting boundary relations are accumulated in host memory, where the CPU performs the boundary matrix reduction. TopoGPU introduces a stratification-aware discrete Morse matching that provably preserves persistent homology under streaming, together with a parallel topological sorting algorithm and a parallel \(V\)-path parity algorithm for deriving Morse boundaries on the GPU. TopoGPU outperforms Cubical Ripser, a state-of-the-art method for persistent homology computation, on every benchmark evaluated, achieving an average end-to-end speedup of $53.24\times$ and a maximum of $198.01\times$. We further integrate TopoGPU into a topology-preserving deep network, demonstrating that it substantially reduces the cost of persistent homology computation during network training. TopoGPU is open source, with pre-built binaries, Google Colab notebooks, and Docker images available at the project's GitHub page: \href{https://github.com/seravee08/GPU-Computation-of-Persistent-Homology-for-Image-Data}{\nolinkurl{https://github.com/seravee08/GPU-Computation-of-Persistent-Homology-for-Image-Data}}.

\end{abstract} 


\section{Introduction}
\label{sec:intro}

Topological data analysis (TDA) provides a principled way to capture the shape of data using topological invariants, and its integration with deep learning has given rise to topological deep learning (TDL) \cite{Zia2024TDL}. A recent review notes that TDL combines the robustness and noise-invariance of persistent homology with the expressive power of neural networks, using topological information for tasks such as classification, regression, clustering, and anomaly detection \cite{Zia2024TDL}. One major line of work uses persistent homology-based descriptors as additional features that encode multiscale connectivity for downstream predictors \cite{Cang2017TopologyNet, Dey2019, wang2021}. A second line develops topology-preserving losses and regularizers that constrain networks to respect target topology during training \cite{hu2019, Chen2019TopoReg, wang2020, hu2021, Clough2022, gupta2025}. A third line integrates topological summaries directly into deep models, including architectures that incorporate topological signatures into graph neural networks, transformers, and autoencoders \cite{Zia2024TDL}.

As a cornerstone of TDA, persistent homology quantifies how topological features such as connected components, loops, and voids appear and disappear across scales and summarizes them as barcodes or persistence diagrams. This evolution of topological features is captured within a series of nested spaces, known as a filtration. Selecting the appropriate filtration depends on the data type: Alpha-shape filtration for 3D point clouds \cite{Edelsbrunner1994}, Vietoris–Rips for high-dimensional metric data \cite{Edelsbrunner2010}, and cubical filtration for 2D and 3D grayscale images \cite{phbook2004, Wagner2012}. This work focuses on the computation of persistent homology for 2D and 3D images, a standard representation across scientific and engineering applications. TopoGPU focuses exclusively on cubical complexes induced by 2D and 3D images; it does not target point clouds, Vietoris-Rips complexes, or general simplicial complexes. Such images may originate from (i) acquisition processes (e.g., CT scans in medical imaging), (ii) numerical computations (e.g., scalar fields from physical simulations), and (iii) intermediate outputs of convolutional neural networks, notably likelihood maps.

While alternative topological descriptors such as merge trees \cite{Nigmetov2019}, contour trees \cite{Gueunet2016}, connected components \cite{Isenburg2011}, and Euler characteristic curves \cite{wagner2017, gpuecc} offer analytical value, they have clear limitations: the first three omit higher-dimensional topology, whereas Euler characteristic curves do not preserve topological features' persistence information. Persistent homology addresses these limitations but introduces significant computational cost. Despite notable advances, persistent homology computation remains a bottleneck during network training; this stands in sharp contrast to modern graphics processing unit (GPU)-optimized deep learning workflows. These workflows keep tensors and computation resident on the device (GPU and its memory) for efficiency. Offloading intermediate tensors (e.g., likelihood maps) to the host (CPU and RAM) for persistent homology computation incurs substantial PCIe transfer and synchronization overhead. This motivates a GPU-based persistent homology pipeline that minimizes host-device transfers and keeps most computations on the GPU.

\IEEEpubidadjcol

Our goal with TopoGPU is twofold: (i) to deliver substantial performance gains and an efficient, reliable computational tool for TDA on 2D/3D image data, and (ii) to catalyze broader adoption of persistent homology within deep learning pipelines. The core novelty of TopoGPU lies in its algorithmic design: a stratification-aware Morse matching with a correctness guarantee under streaming, and parallel algorithms that make Morse boundary derivation feasible on GPUs; the streaming pipeline and GPU implementation realize this design in practice. Our main contributions are as follows.

\begin{itemize}
    \item We present TopoGPU, a massively parallel GPU algorithm for persistent homology computation of cubical complexes induced by 2D/3D grayscale images. It is faster than Cubical Ripser (CRipser) \cite{kaji2020}, a state-of-the-art method, on every evaluated benchmark, with average and maximum end-to-end speedups of $53.24\times$ and $198.01\times$.

    \item We propose stratification-aware discrete Morse matching that preserves the persistent homology of the input under a streaming scheme.

    \item We propose a parallel topological sorting algorithm for bounded-degree Morse matching graphs that achieves a 46.13$\times$ average and up to 54.81$\times$ speedup over classic Kahn's algorithm.

    \item We propose a parallel \(V\)-path parity algorithm on the Morse matching graphs, running in time linear in the graph size under a parallel model with sufficient thread parallelism. Together with our parallel topological sorting algorithm, it enables efficient derivation of boundary relations of the Morse complex.

    \item We integrate TopoGPU into a topology-preserving deep learning framework, demonstrating that persistent homology computation is no longer the dominant cost during network training.
\end{itemize}

\section{Background}
\label{sec:background}
This section introduces the key concepts our GPU algorithm relies on and uses directly in Section~\ref{sec:algo}.

\noindent \textbf{Cubical complex construction from an image.} This work focuses on grid-structured data, specifically 2D and 3D images. Consider a 3D image of dimensions \((W, H, D)\) represented by the function \( g: \Omega \to \mathbb{R} \), where \( \Omega \) is a rectangular subset within \( \mathbb{Z}^3 \) spanning a discrete lattice with ranges \( 0 \leq x < W \), \( 0 \leq y < H \), and \( 0 \leq z < D \) along the \(x\)-, \(y\)-, and \(z\)- axes, respectively. While spatial coordinates in \(\Omega\) correspond to pixels in 2D and voxels in 3D, we adopt  the term \textit{voxel} throughout this work for consistency.

Building on this discrete lattice, we model the image domain as a cubical complex. Following the formalism of \cite{phbook2004}, we classify intervals within \(\mathbb{R}\) as either \textit{degenerate} \([k, k]\) or \textit{regular} \([k, k+1]\), where \(k \in \mathbb{N}\). A \(p\)-dimensional \textit{cubical cell}, or simply \(p\)-cell, within a \(d\)-dimensional space is formed by the product of \(d\) such intervals, of which \(p\) are regular. The dimension of a \(p\)-cell \(\alpha\) is denoted \(\dim(\alpha) = p\). A \textit{cubical complex} \(\mathcal{K}\) is the collection of all \(p\)-cells generated from the lattice \(\Omega\), with \(0 \leq p \leq 2\) in 2D and \(0 \leq p \leq 3\) in 3D.

A cubical complex can be constructed from an image in two ways: the \( T \)-construction which maps voxels to top-dimensional cells, and the \( V \)-construction which maps voxels to \(0\)-cells (vertices) \cite{vanessa2011-morsematch, wagner2017, Bleile2022}. These two constructions are dual and convey identical topological information. This work adopts the \( T \)-construction, wherein \(0\)-cells are vertices, \(1\)-cells are edges, \(2\)-cells are squares (equivalent to pixels in 2D), and \(3\)-cells are cubes (equivalent to voxels in 3D). In a cubical complex \(\mathcal{K}\), a \(p\)-cell \(\alpha\) is a \textit{face} of a \(q\)-cell \(\beta\) if \(\alpha \subseteq \beta\); it is a \textit{proper face} if, in addition, \(p < q\), denoted \(\alpha \subset \beta\). Conversely, we call \(\beta\) a \textit{(proper) coface} of \(\alpha\).

In the \(T\)-construction in a \(d\)-dimensional space, each top-dimensional cell \(\beta\) inherits its value \(f(\beta)\) from the corresponding input voxel. These values propagate downward to lower-dimensional faces, where each cell \( \alpha \) adopts the minimum value among its top-dimensional cofaces, formalized as \( f(\alpha) = \min\{f(\beta): \alpha \subset \beta \text{ and } \dim(\beta) = d\} \). We adapt the term ``lower star'' of a top-dimensional cell \(\beta\) to denote the set of all cells for which \(\beta\) is a minimal-valued coface, defined as:
\begin{equation}
L(\beta) = \left\{ 
\alpha \in \mathcal{K} 
\mid 
\alpha \subseteq \beta 
\text{ and } 
f(\beta) = \min_{\substack{\alpha \subseteq \gamma \\ \dim(\gamma) = d}} f(\gamma)
\right\}
\label{eq:lower_star}
\end{equation}
We informally refer to these cells as being \textit{introduced} by \(\beta\). The lower stars of all voxels \(\beta \in \Omega\) form a disjoint partition of \(\mathcal{K}\). In practical settings where voxels share identical values, a cell \(\alpha\) may have multiple top-dimensional cofaces sharing the same minimum value, creating a tie regarding which coface introduces \(\alpha\). To resolve such ties, we impose a lexicographical ordering on the 3D grid \(\Omega\) with dimensions \( W \times H \times D\), defined by the indexing function:
\begin{equation}
\phi(x,y,z) = (z \cdot W \cdot H) + (y \cdot W) + x
\label{eq:val_index}
\end{equation}
This corresponds to the row-major memory layout standard in C/C++. The ordering serves as a tie-breaker by selecting the coface with minimal \(\phi\)-value in case of a tie. Figure~\ref{fig:morsematch:b} visualizes the resulting cubical complex derived from Figure~\ref{fig:morsematch:a}, where cells within the same lower star are identically colored.

\vspace{2mm}
\noindent \textbf{Persistent homology.} In algebraic topology, \textit{homology groups} capture basic topological structures of a space \cite{Edelsbrunner2010}. For a 3D cubical complex, the homology group in dimension 0 (\(H_0\)) counts connected components, the homology group in dimension 1 (\(H_1\)) detects loops, and the homology group in dimension 2 (\(H_2\)) detects enclosed voids. This is illustrated in Figure~\ref{fig:morsematch:e} (x-axis, from left to right: \(H_0\), \(H_1\)). The count of topological structures within each dimension is quantified by the \(p^{\text{th}}\) \textit{Betti number}, defined as the rank of the \(p^{\text{th}}\) homology group. In this study, we work with homology and persistent homology with \(\mathbb{Z}_2\) coefficients. Formal definitions of chains, cycles, boundary operators, and homology groups are omitted; see \cite{hatcher2011,Munkers84} for details.

A \textit{subcomplex} of \(\mathcal{K}\) at threshold \(t\), denoted \(\mathcal{K}_t\), consists of all cells \(\alpha \in \mathcal{K}\) with values \(f(\alpha) \leq t\). For \(t < \min_{\alpha \in \mathcal{K}} f(\alpha)\), \(\mathcal{K}_t\) is empty. As \(t\) increases, for all \(t_1 \leq t_2\) we have \(\mathcal{K}_{t_1} \subseteq \mathcal{K}_{t_2}\), forming a nested sequence of subcomplexes. Eventually, when \(t\) is at least \(\max_{\alpha \in \mathcal{K}}f(\alpha)\), the entire complex is included: \(\mathcal{K}_t = \mathcal{K}\). This nested family of subcomplexes \(\{\mathcal{K}_t\}_{t\in \mathbb{R}}\) forms a \textit{filtration}. Persistent homology \cite{Edelsbrunner2010} studies the evolution of topological features across this filtration.

Specifically, adding a \(p\)-cell \(\alpha\) to a subcomplex either creates a new \(p\)-dimensional topological structure or kills a \((p-1)\)-dimensional one \cite{Edelsbrunner2010}. Persistent homology records the \textit{birth} of a topological structure at threshold \(b\) when it first appears in the filtration and its \textit{death} at threshold \(d\) when it merges with an older structure or becomes a boundary; the pair \((b, d)\) is called a \textit{persistence pair}. The collection of all such pairs forms a \textit{persistence diagram}, a multiset of points \((b, d) \in \mathbb{R}^2\) with \(d > b\); the persistence of a feature is \(d - b\). Figure~\ref{fig:morsematch:e} illustrates a persistence diagram where a loop is born at 10 and filled at 12, yielding the pair \((10, 12)\) with persistence 2.\\

\noindent \textbf{Discrete Morse theory.} 
Introduced by Forman \cite{forman2002} in the 1990s, discrete Morse theory is a combinatorial adaptation of Morse theory \cite{forman1998} to real-valued functions defined on cell complexes.

A discrete \textit{Morse matching} on a filtered cubical complex \(\mathcal{K}\) computes a collection of disjoint pairs of cells (\(\alpha, \beta\)), where each cell appears in at most one pair. Specifically, a \(p\)-cell \(\alpha\) is matched to a \((p - 1)\)-dimensional proper face \(\beta\) of \(\alpha\), with matchings constrained to cells within the same lower star. This matching process induces a directed acyclic graph (DAG) called the \textit{matching graph} or, equivalently, the \textit{gradient vector field}.

Figures~\ref{fig:morsematch:b} and~\ref{fig:morsematch:c} depict the graphs before and after discrete Morse matching, where each node corresponds to a cell in the cubical complex \(\mathcal{K}\). Prior to matching, solid arrows point from each cell to its proper faces. Upon matching, these arrow directions invert: dashed arrows now point from cells to their proper cofaces. Unpaired cells are called \textit{critical cells}. Critical cells are highlighted in red in Figures~\ref{fig:morsematch:c} and ~\ref{fig:morsematch:d}. A \textit{\(V\)-path} is defined as a path in the matching graph that begins at a \(p\)-cell, alternates between \(p\)- and \((p - 1)\)-cells, and ends at a \((p - 1)\)-cell.

A discrete \textit{Morse complex} induced by a matching is generated by its critical cells. When working with \(\mathbb{Z}_2\) coefficients (binary arithmetic modulo 2), boundary computations reduce to parity checks: a critical \((p-1)\)-cell \(\alpha\) lies in the boundary of a critical \(p\)-cell \(\beta\) if and only if there is an \textbf{odd number of \(V\)-paths} from \(\beta\) to \(\alpha\). In Figure~\ref{fig:morsematch:c}, red arrows highlight a single \(V\)-path from a critical 2-cell to a critical 1-cell, so the 1-cell lies in the boundary of the 2-cell.

These boundary relations are encoded in a filtered \textit{boundary matrix} \(M\) with entries \(M(i, j) = 1\) iff a \((p-1)\)-dimensional critical cell \(\sigma_j\) lies in the boundary of a \(p\)-dimensional critical cell \(\sigma_i\). The discrete Morse complex inherits filtration values from its underlying critical cubical cells. After ordering \(M\)'s rows and columns by their cell filtration values, persistent homology is computed from boundary matrix reduction via Gaussian elimination over \(\mathbb{Z}_2\) \cite{Edelsbrunner2010, phat}.

For a cubical complex \(\mathcal{K}\) with \(n\) cells, its boundary matrix \(M'\) is an \(n \times n\) binary matrix with \(M'(i, j) = 1\) iff \((p-1)\)-cell \(\sigma_j\) is a face of \(p\)-cell \(\sigma_i\). Crucially, the boundary matrix of the Morse complex is typically much smaller than that of the cubical complex. This reduction in matrix size offers significant computational savings while yielding homology groups isomorphic to those of \(\mathcal{K}\), thereby preserving topological equivalence.

\section{Related Work}
\label{sec:relatedwork}

\begin{figure*}[!t]
    \centering
    \captionsetup[subfigure]{font=small}
    \subfloat[{\scriptsize Input image}\label{fig:morsematch:a}]{\includegraphics[width=0.13\textwidth]{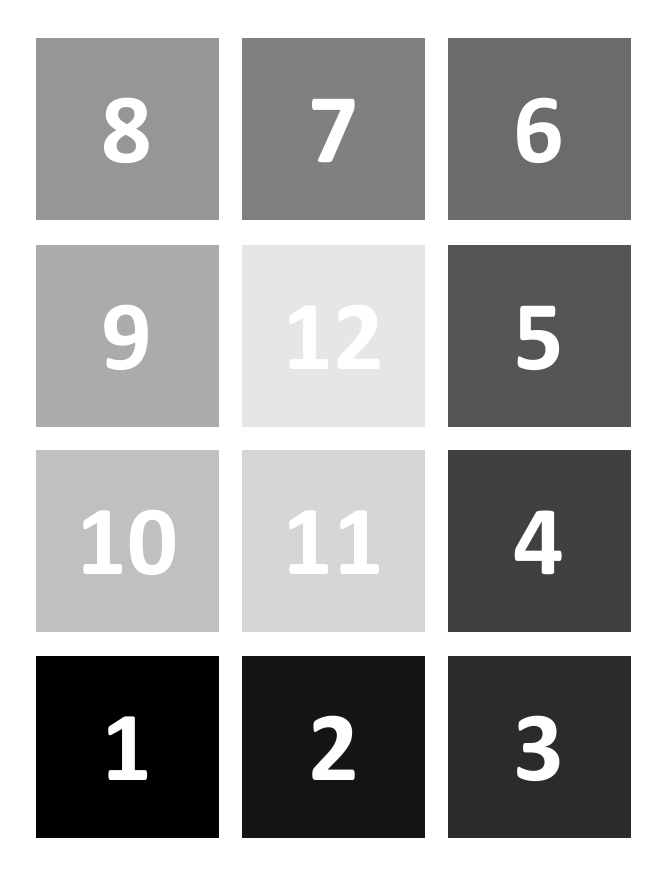}} 
    \hfill
    \subfloat[{\scriptsize Filtered cubical complex}\label{fig:morsematch:b}]{\includegraphics[width=0.20\textwidth]{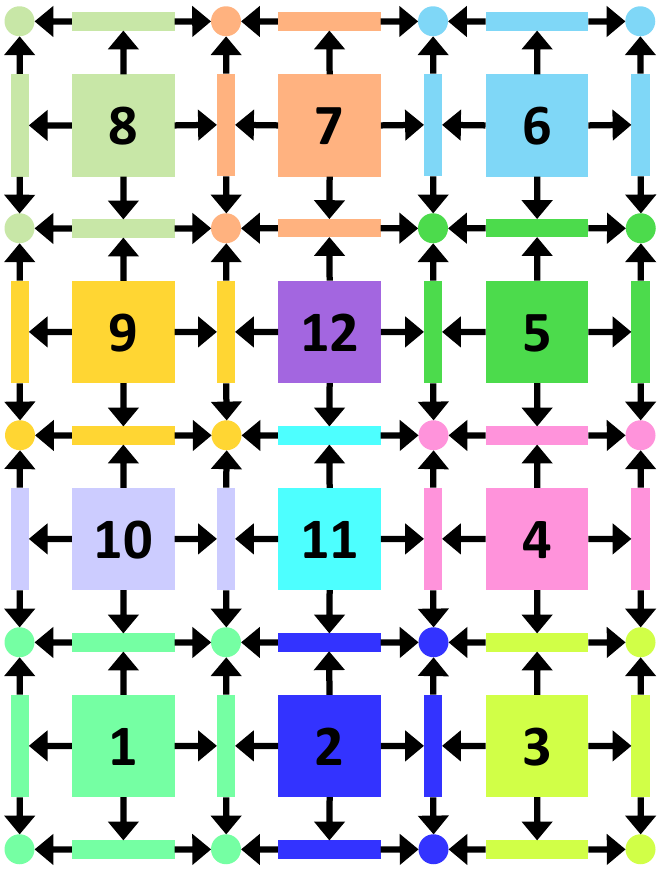}} 
    \hfill
    \subfloat[{\scriptsize Matching graph}\label{fig:morsematch:c}]{\includegraphics[width=0.20\textwidth]{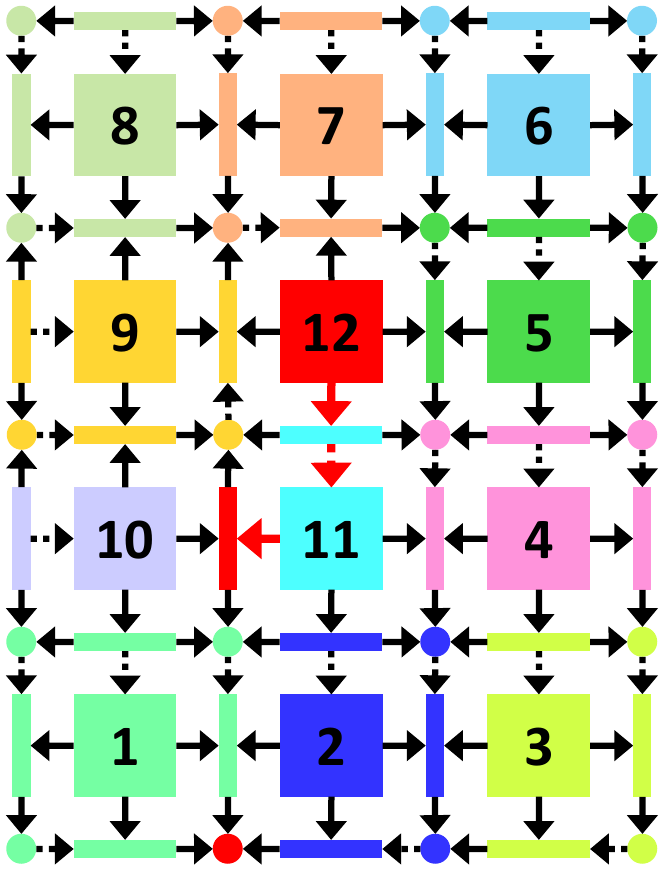}} 
    \hfill
    \subfloat[{\scriptsize Stratification-aware matching}\label{fig:morsematch:d}]{\includegraphics[width=0.21\textwidth]{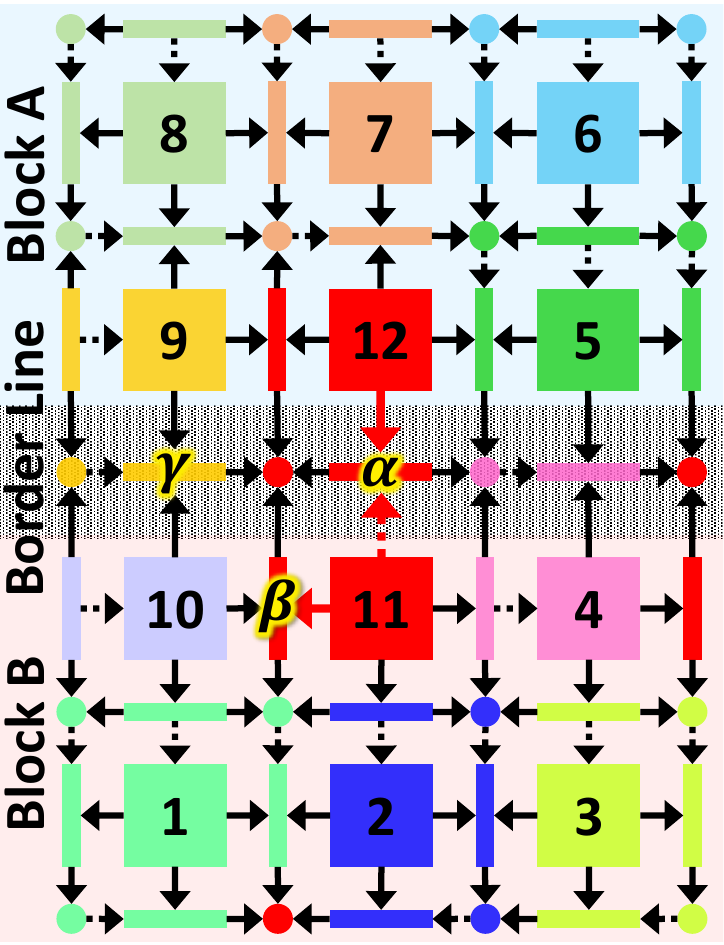}}
    \hfill
    \subfloat[{\scriptsize Persistence diagram}\label{fig:morsematch:e}]{\includegraphics[width=0.23\textwidth]{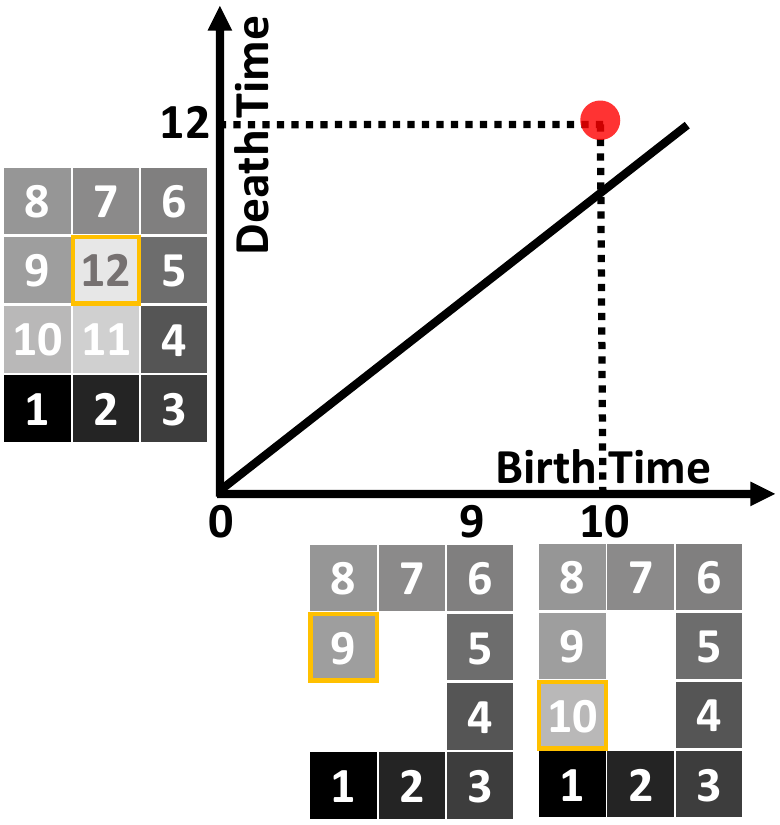}} 
    
    \caption{(a) Input image; (b) Filtered cubical complex of (a) with solid arrows initially pointing from cells to their proper faces; (c) Matching graph of (b), with dashed arrows indicating Morse matchings, critical cells in red, and red arrows indicating a \(V\)-path; (d) Stratification-aware Morse matching after splitting (a) into two blocks, with border region shaded; (e) Persistence diagram encoding boundary relations with a persistence pair (\(10, 12\)).}
    \label{fig:morsematch} 
\end{figure*}

In this section, we review related work on persistent homology and discrete Morse complex for image data. Verri et al. \cite{verri1993} pioneered the application of algebraic topology to image data via size functions. Mischaikow and Mrozek \cite{konstantin1995} broadened the scope by employing the Conley index to analyze dynamics in nonlinear systems; this framework later enabled Morse-theoretic computational methods \cite{phbook2004}.

The early 2010s witnessed significant advances in topological analysis of image data. Bendich et al. \cite{paul2010} proposed an efficient approximation scheme for persistent homology computation on 3D image data. Wagner et al. \cite{Wagner2012} developed an efficient boundary matrix generation framework for exact persistence computation on image data of arbitrary dimensions. This approach was later implemented in GUDHI \cite{gudhi} and DIPHA \cite{ulrich2014}. Despite these advances, boundary matrix construction remains a computational bottleneck, necessitating further optimization.

Lewiner et al. \cite{LEWINER2003} introduced an algorithm to build a discrete Morse function on a mesh that minimizes the number of critical cells, yielding a Morse complex for downstream homology computation. Using discrete Morse theory, Robins et al. \cite{vanessa2011-morsematch} marked a breakthrough in computing topological descriptors for image data: they constructed a Morse complex that preserves the persistent homology of the original image while being substantially smaller than the cubical complex used in persistence computations. G{\"u}nther et al. \cite{david2012} introduced a memory-efficient storage of Morse matchings that accommodates large volumes, but it remains fast primarily on relatively simple data and comes at the cost of notably increased implementation complexity. TopoGPU can handle large datasets using a streaming scheme that keeps only one chunk in device memory at a time, and it can process arbitrarily large inputs provided the resulting boundary matrix fits in host memory.

Mischaikow and Nanda \cite{konstantin2013} extended the framework of Robins et al. \cite{vanessa2011-morsematch} beyond cubical complexes and introduced the Perseus software for persistence computation across diverse data types. Their approach prioritized flexibility and broad applicability, with a trade-off in computational efficiency. Additionally, they introduced the first streaming precomputation approach, where individual level sets are stored, simplified independently, and subsequently merged, contrasting with our spatial volume decomposition strategy.

Wagner proposed Cubicle \cite{chunking}, an efficient variant of \cite{vanessa2011-morsematch} that divides large image data into slices and uses streaming preprocessing to enable scalable CPU processing. 
We adopt a more aggressive, hierarchical data partitioning strategy: the input is divided into chunks that are further subdivided into GPU blocks, enabling per-voxel processing with dedicated GPU threads. 
Both approaches can be viewed in terms of a stratified version of discrete Morse theory developed by Knudson and Wang~\cite{knudson2022discrete}, and the version for open complexes~\cite{knudson2026discrete} developed by Knudson and Scoville.

Beyond persistence computation, discrete Morse theory has enabled visualization-oriented simplifications; for example, Gyulassy et al. \cite{attila2008} processed billion-voxel volumes on standard hardware within one day. However, the resulting simplified complexes cannot be used to compute persistent homology of the original data \cite{Olaf2015}.

Bauer's Ripser \cite{ripser}, initially designed for Vietoris–Rips complexes on point cloud data, was extended to process image data via CRipser \cite{kaji2020}, a CPU implementation that serves as our primary baseline. Another popular open-source platform, the Topology ToolKit (TTK) \cite{ttk}, provides a comprehensive suite of CPU-based topological analysis and visualization algorithms widely used in scientific workflows, whereas TopoGPU focuses on GPU-accelerated persistent homology computation for image data. \texttt{pytorch-topological} \cite{rieck_pytorch_topological} is a PyTorch-based library for topological machine learning that computes persistent homology of various filtrations, including cubical complexes via a GUDHI backend, and provides differentiable topological layers and losses for deep learning. TopoGPU is distinguished from these packages by its computational model: GUDHI, DIPHA, and CRipser are CPU-based and non-streaming, while Cubicle streams on the CPU but offers no GPU parallelism and no floating-point support. In contrast, TopoGPU is the first to carry the entire precomputation, from cubical complex construction through Morse boundary derivation, on the GPU under a streaming scheme, leaving only the boundary matrix reduction on the CPU.

As modern GPUs have become more powerful, there has been growing interest in developing fast, memory-efficient parallel workflows for large-scale topological analysis. Zhang et al. \cite{zhang2020gpu} introduced Ripser++, a GPU-accelerated algorithm optimizing Vietoris-Rips computations. Wang et al. \cite{gpuecc} exploited GPU parallelism to compute cubical complexes and Euler characteristic curves on image data. gMSC \cite{Subhash2023} mitigated the path counting bottleneck in Morse-Smale complex computation via a GPU algorithm converting graph traversals into concurrently executable tasks. In this work, we tackle the path counting challenge with a parallel topological sorting method coupled with a parallel \(V\)-path parity computation algorithm. Recently, CuMPerLay \cite{cmp} introduced a CUDA implementation of cubical persistence for 2D images as a differentiable layer, computing 0D and 1D persistence via union-find, with the latter obtained through Alexander duality. This 2D-specific reduction does not extend to 3D volumes, which TopoGPU targets across all homology dimensions.

\section{Our Algorithms}
\label{sec:algo}
GPU-based persistent homology computation is constrained by per-block resource limits and the overhead of inter-block coordination. A GPU block, hereafter simply referred to as a block, is a cooperative group of threads that can synchronize and share on-chip memory. On the NVIDIA RTX 4080 used in our experiments, the per-block limits are up to 1,024 threads, 99 KB shared memory, and 64 K 32-bit registers. This limits the efficiency of single-block computation, while multiple blocks introduce expensive communication overhead. Additionally, modern datasets often exceed the memory capacity of a single GPU or CPU. 

Together, these constraints motivate a GPU algorithm capable of (1) decomposing the computational task into independent subtasks executable per GPU block, (2) eliminating communication dependencies across GPU blocks, and (3) streaming the input and aggregating results from partial outputs. Robins et al.'s \textit{ProcessLowerStars} \cite{vanessa2011-morsematch} addresses (1) and (2): it computes a matching graph independently within each voxel's lower star. This yields a computation pattern that maps naturally to GPU frameworks: one thread per voxel computes its local matching in parallel, with no dependencies across blocks. Wagner \cite{chunking} addresses (3) by partitioning the input data into chunks and aggregating each chunk's partial results into a global boundary matrix.

Building on these ideas, we present a massively parallel GPU streaming algorithm for persistent homology on 2D/3D grayscale image data that outputs persistence diagrams. The input image is streamed and processed chunk by chunk, and the computation proceeds through the following stages:

\begin{itemize}
    \item \textbf{Preprocessing} (Section~\ref{subsec:preproc}).
    A chunk is extracted from the input image and padded with collars. We launch a grid of non-overlapping GPU blocks to cover the chunk, excluding the collars. These blocks compute Morse matchings and discrete Morse boundaries independently and in parallel. Once all blocks in a chunk complete, the next chunk is loaded and processed similarly.

    \item \textbf{Morse matching computation} (Section~\ref{subsec:morsematch}).
    For each block, we construct its matching graph in two steps:
    
    \begin{itemize}
        \item \textbf{Cubical complex construction} (Section~\ref{subsubsec:cplx}).
        We first construct a filtered cubical complex. The filtered cubical complex corresponding to the input in Figure~\ref{fig:morsematch:a} is shown in Figure~\ref{fig:morsematch:b}.
        
        \item \textbf{Stratification-aware Morse matching} (Section~\ref{subsubsec:gpuborder}).
        Subsequently, following \cite{vanessa2011-morsematch}, we compute a discrete Morse matching to obtain a matching graph. To ensure correct persistent homology computation in our streaming setting, we introduce a stratification-aware Morse matching that explicitly handles matchings at block borders (Figure~\ref{fig:morsematch:d}). This approach introduces additional critical cells compared to the matching performed on the entire cubical complex (Figure~\ref{fig:morsematch:c}), which yields the optimal number of three critical cells. However, these extra critical cells are subsequently used to \emph{stitch} the results of per-block local computations.
    \end{itemize}

    \item \textbf{Discrete Morse boundary computation} (Section~\ref{subsec:morsebound}).
    With the matching graph constructed for each block, we derive the boundary relations of the resulting Morse complex by counting \(V\)-paths between critical cells. While explicitly enumerating all \(V\)-paths can be exponential in the graph size, over \(\mathbb{Z}_2\) the boundary relations depend only on the parity (odd/even) of the number of such paths, reducing the problem from exponential-time path enumeration to a linear-time parity computation. We realize this boundary derivation on the GPU via two algorithms:
    
    \begin{itemize}
        \item \textbf{Parallel topological sorting} (Section~\ref{subsubsec:tpsort}).
        We first compute a topological order of the matching graph to provide the linear order required for boundary derivation.
        
        \item \textbf{Parallel path parity computation} (Section~\ref{subsubsec:pathparity}).
         We then traverse that order with a parallel path parity algorithm that tracks cumulative path parities between critical cells, thereby deriving the boundary relations of the Morse complex.
    \end{itemize}

    Figure~\ref{fig:morsematch:c} illustrates a \(V\)-path from a critical 2-cell (value 12) to a critical 1-cell (value 10), indicating that the 1-cell lies in the boundary of the 2-cell.

    \item \textbf{Boundary matrix reduction on CPU} (Section~\ref{subsec:bndmatreduction}).
    After blocks across all chunks have been processed and the complete set of boundary relations has been finalized in host memory, critical cells are first sorted by filtration value, after which the boundary matrix is constructed and reduced to compute persistent homology. The persistence pair induced by the boundary relation in Figure~\ref{fig:morsematch:c} appears in the persistence diagram in Figure~\ref{fig:morsematch:e}.

\end{itemize}

\begin{figure*}[!t]
\centering
\includegraphics[width=6.0in]{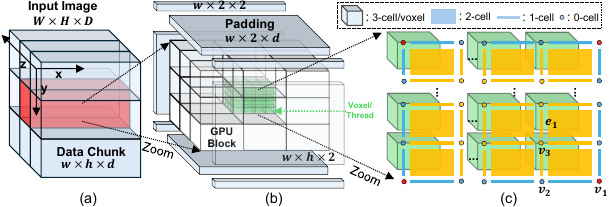}
\caption{(a) A 3D input partitioned into chunks; (b) collar-padded chunk tiled by a grid of non-overlapping GPU blocks, with one GPU thread per voxel; (c) front \(2\)-dimensional border complex of the green GPU block in (b), surrounded by 26 neighboring blocks.}
\label{fig:bordermanip}
\end{figure*}

\subsection{Preprocessing}
\label{subsec:preproc}
The input image is partitioned as follows: 2D images along the y-axis; 3D images along both y- and z-axes (see Figure~\ref{fig:bordermanip}a). A 3D image of size \((W, H, D)\) is cut \(c_y\) times along y-axis and \(c_z\) times along z-axis, yielding chunks of size \((w, h, d) = (W, H/c_y, D/c_z)\). This layout respects C++ row-major storage and maximizes contiguous memory access within each chunk.

Each chunk is padded with a two-voxel-wide collar drawn from adjacent chunks: 2D chunks of dimensions \((w, h)\) become \((w, h + 4)\), and 3D chunks of dimensions \((w, h, d)\) become \((w, h + 4, d + 4)\). Collar values are sourced from adjacent chunks when available, or set to \(+\infty\) if no adjacent chunk exists. No padding is applied along the x-axis, as data is not partitioned in this direction. Each padded chunk is loaded with the C++ Thread Pool Library (CTPL) using parallel CPU threads and stored in a single device buffer. Figure~\ref{fig:bordermanip}b shows a padded 3D chunk. The role of collars is discussed in Section~\ref{subsubsec:gpuborder}.

\subsection{Morse Matching Computation}
\label{subsec:morsematch}

\subsubsection{Cubical complex construction}
\label{subsubsec:cplx}
Each GPU block is a group of threads organized in a 2D/3D structure mirroring the spatial layout of pixels/voxels in the input data. This alignment ensures each voxel is assigned exactly one GPU thread, establishing a one-to-one correspondence between threads and data elements (see Figures~\ref{fig:bordermanip}b-c for a 3D example).

Each GPU thread processes its voxel by (i) identifying its lower star (Equation~\ref{eq:lower_star}) and (ii) computing the matching graph for the derived lower star. During lower star determination, each thread accesses a neighboring voxel by adding a fixed offset to its own coordinate and reading the corresponding buffer entry via Equation~\ref{eq:val_index}. Since the lower stars of all voxels form a disjoint partition of the cubical complex, GPU threads execute independently in parallel. Cells in a lower star inherit their filtration value from the corresponding introducing voxel. Therefore, we do not explicitly store filtration values for the cubical complex; instead, they are determined on the fly from introducing voxels during Morse matching computation. After the matching graphs for all blocks in the current chunk have been constructed, the device buffer storing the chunk values is reused to load subsequent chunks.

\subsubsection{Stratification-Aware Morse Matching}
\label{subsubsec:gpuborder}
In \cite{chunking}, the input cubical complex is partitioned along a single dimension, yielding a decomposition into two types of strata: interior strata and border strata, where each border stratum is a codimension-\(1\) cubical complex. We adopt a more general data partitioning scheme and therefore require a more refined stratification. Let \(\mathcal{K}\) be the \(D\)-dimensional cubical complex of a block of voxels. We distinguish \(D + 1\) strata, indexed by dimension \(d \in \{0, 1, \ldots, D\}\). For each \(d\), we define the \(d\)-dimensional \textit{stratum} as the collection of all cells that belong to the cubical subcomplexes of exactly \(2^{D - d}\) blocks.

The stratum of dimension \(D\) is called the interior stratum, and the strata of dimensions \((0, 1, \ldots, D - 1)\) are collectively referred to as border strata. While the strata form a disjoint decomposition of \(\mathcal{K}\), an individual \(d\)-dimensional stratum is generally not a cubical complex. We define the \(d\)-dimensional \textit{border complex} as the union of all \(k\)-dimensional strata with \(k \le d < D\). Each \(d\)-dimensional border complex is a cubical subcomplex of \(\mathcal{K}\), being closed under taking faces.

Consider the green 3D block in Figure~\ref{fig:bordermanip}b and its front \(2\)-dimensional border complex in Figure~\ref{fig:bordermanip}c. The \(0\)-dimensional border strata consist of the \(0\)-cells shared by the cubical complexes of eight blocks (red cells); the \(1\)-dimensional border strata consist of the \(0\)- and \(1\)-cells shared by the cubical complexes of four blocks (blue cells); and the \(2\)-dimensional border strata consist of the \(0\)-, \(1\)-, and \(2\)-cells shared by the cubical complexes of two blocks (yellow cells).

A border-crossing path is an alternating path that contains cells belonging to two or more interior strata. Stratification-aware Morse matching prevents such paths by introducing auxiliary cells, a subset of which are critical, thereby enabling correct aggregation of local results within the streaming framework. Recall that each chunk is covered by a grid of non-overlapping GPU blocks (Figure~\ref{fig:bordermanip}b). Since preventing alternating paths from crossing blocks also prevents them from crossing chunks, we focus on border-crossing paths across GPU blocks. Despite additional auxiliary critical cells, the resulting boundary matrix remains considerably smaller than that of the original complex.

We now introduce stratification-aware Morse matching. For a filtered cubical complex decomposed into GPU blocks, discrete Morse matchings are computed independently within each block's interior stratum and within its border strata of each dimension. Matchings are computed among cells that lie in the same lower star and the same stratum, following \textit{ProcessLowerStars}~\cite{vanessa2011-morsematch}. The union of these matchings forms the \textit{global matching}, and the corresponding filtered discrete Morse complex is called the \textit{global discrete Morse complex}. Any cell that remains unpaired after this process is critical. In particular, cells in \(0\)-dimensional border strata are always critical as there are no cells available for pairing within the same stratum. 

Under stratification-aware Morse matching, once a \(V\)-path reaches a border complex, it cannot leave; moreover, within a border complex, \(V\)-paths can only remain in the same-dimensional border stratum or move to lower-dimensional border strata. Accordingly, we restate Lemma~1 of \cite{chunking} as follows:

\begin{lemma}[Border Blocking Lemma]
\label{lemma:borderblocking}
Let \(G\) be the global matching computed on a filtered cubical complex decomposed into blocks, and let \(B\) be any border complex in this decomposition. Consider any directed path \(p = (p_1, p_2, \ldots, p_i, \ldots, p_n)\) in \(G\). If \(p_i \in B\) for some \(i\), then \(p_j \in B\) for all \(j \geq i\).
\end{lemma}

\begin{proof}
We show that if $p_k \in B$ then $p_{k+1} \in B$; the claim then follows by induction from $p_i \in B$. In the global matching graph $G$, the vertex $p_k$ has two kinds of outgoing edges: (i) an edge to a proper face of $p_k$, which lies in $B$ since $B$ is a cubical complex; and (ii) a flipped edge to the proper coface with which $p_k$ is matched, which also lies in $B$ because stratification-aware matchings pair only cells within the same stratum. In either case $p_{k+1} \in B$, and by induction $p_j \in B$ for all $j \geq i$.
\end{proof}

Since matchings within each stratum are acyclic by construction and Lemma~\ref{lemma:borderblocking} forbids any cycle spanning multiple strata, the global matching is acyclic and hence a valid discrete Morse matching on the entire complex, reducing correctness to the non-partitioned case \cite{chunking}:

\begin{corollary}[Global Correctness]
The global discrete Morse complex encodes the same persistent homology up to zero-persistence pairs.
\end{corollary}

For example, the \(V\)-path indicated by the red arrows in Figure~\ref{fig:morsematch:c} becomes a border-crossing path from Block A to Block B after the input is divided into two blocks. Stratification-aware Morse matching inserts auxiliary critical cells and splits this path into three paths: one from cell 12 to \(\alpha\) in Block A, and two from cell 11 to \(\alpha\) and \(\beta\) respectively in Block B (Figure~\ref{fig:morsematch:d}). Thus, \(\alpha\) lies in the boundaries of cells 12 and 11, while \(\beta\) is in the boundary of cell 11. During \(\mathbb{Z}_2\) boundary matrix reduction, adding the column of cell 11 to that of cell 12 cancels \(\alpha\) and introduces \(\beta\) into the boundary of 12, preserving the original persistence pair \((10, 12)\). The introduced auxiliary critical cells contribute only zero-persistence pairs.

Since matchings are restricted to cells in the same stratum and the same lower star, cells within a given stratum may be processed in lower stars associated with different blocks. Consider the shaded \(1\)-dimensional border complex in Figure~\ref{fig:morsematch:d}. \(\alpha\) is introduced by cell 11 and processed in Block B, whereas \(\gamma\) is introduced by cell 9 and processed in Block A. This poses no issues when Blocks A and B lie within the same chunk, since all blocks in a chunk collectively operate on a shared matching graph and have access to the matching graph for the entire chunk.

The situation differs when Blocks A and B belong to different chunks: \(\gamma\)'s matchings are computed in Block A and are missing from Block B's chunk, while \(\alpha\)'s matchings are computed in Block B and are missing from Block A's chunk, causing the matching graphs of the two adjacent chunks to disagree on the border strata. To address this, we extend the chunk containing Block A with a two-voxel-wide collar into the neighboring chunk containing Block B, so that Block A can compute \(\alpha\)'s matchings locally within its chunk. The same mechanism applies symmetrically to the chunk containing Block B. These collars ensure that the matching graphs of adjacent chunks induce consistent matchings on border complexes. By \cite{vanessa2011-morsematch}, determining the lower star of a voxel requires its 8-connected neighbors in 2D, or 26-connected neighbors in 3D; resolving a border cell's containing lower star thus reaches one voxel across the chunk border to the introducing voxel and one voxel beyond it, so a collar width of two is the minimum required. Accordingly, each chunk is padded with a two-voxel-wide collar on both sides of every cutting direction.

We do not launch additional blocks or threads to process the collar voxels. Instead, we repurpose the interior threads nearest the chunk border to process them once they finish their own, following the same stratification-aware Morse matching procedure. Since these threads already reside next to the border, the collar values they require are fixed-offset neighbors read directly from the device buffer.

The Morse matching computation outputs a graph \(G = (V, E)\), along with the list of critical cells \(\mathcal{C}\) and their filtration values, where \(V\) is the set of cell indices of the filtered cubical complex and \(E\) is the set of directed edges of the matching graph. We store \(G\) in a cube-map format \cite{Wagner2012}, whose entries index the cells in \(V\) and encode the incident edges in \(E\); cell indices are computed on the fly from cell coordinates in this grid. The cube-map format is a regular grid-like layout that stores one entry for every vertex, edge, face, and voxel of the cubical complex induced by the input grid. For a 3D chunk of size \((w, h, d)\), the cube-map is an array of size \((2w + 1, 2h + 1, 2d + 1)\), where the cell dimension of an entry with index \(\mathbf{i}=(i_1,i_2,i_3)\) is given by \(\dim(\mathbf{i})=\sum_{k=1}^{3}(i_k \bmod 2)\). Figure~\ref{fig:morsematch:b} illustrates the cube-map format induced by the 2D input in Figure~\ref{fig:morsematch:a}. Each entry in \(E\) occupies 12 bits, partitioned into six consecutive 2-bit segments representing front, back, left, right, top, and bottom directions. Each segment encodes one of three matching states: incoming edge, outgoing edge, or disconnected (used in GPU topological sorting).

\subsection{Discrete Morse Boundary Computation}
\label{subsec:morsebound}

\begin{figure*}[!t]
    \captionsetup[subfigure]{font=small}
    \centering
    \subfloat[\centering \(d^-(v)=0\), \(pre(v)=v\) \label{fig:toposortintuit:a}]{\includegraphics[width=0.12\textwidth]{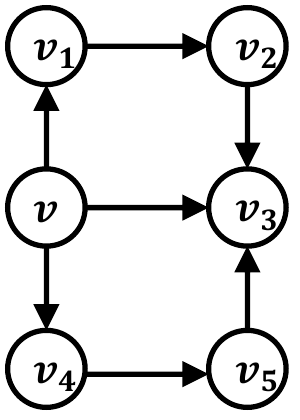}} 
    \hspace{1cm}
    \subfloat[\centering \(d^-(v)=1\), \(pre(v)=v_1\) \label{fig:toposortintuit:b}]{\includegraphics[width=0.12\textwidth]{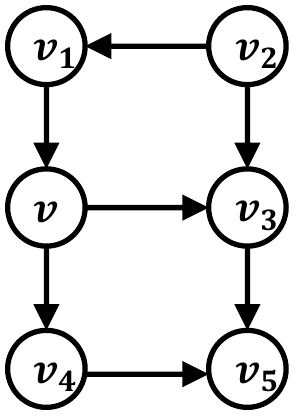}} 
    \hspace{1cm}
    \subfloat[\centering \(d^-(v)=3\), \(pre(v)=v_1\) \label{fig:toposortintuit:c}]{\includegraphics[width=0.12\textwidth]{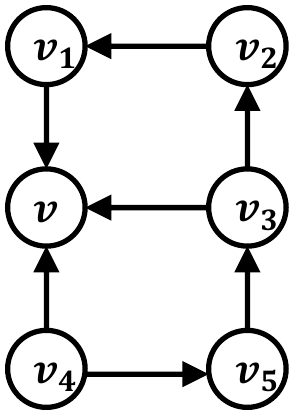}} 
    \hspace{1cm}
    \subfloat[\centering \(d^-(v)=2\), \(pre(v)=\text{\textup{\textsc{nil}}}\) \label{fig:toposortintuit:d}]{\includegraphics[width=0.12\textwidth]{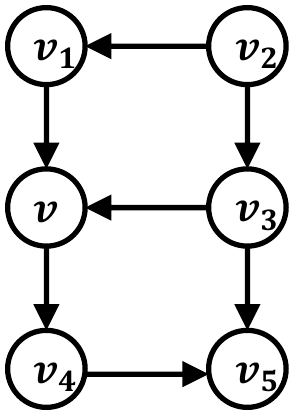}}
    \hspace{1cm}
    \subfloat[\centering \(d^-(v)=2\), \(pre(v)=\text{\textup{\textsc{nil}}}\) \label{fig:toposortintuit:e}]{\includegraphics[width=0.12\textwidth]{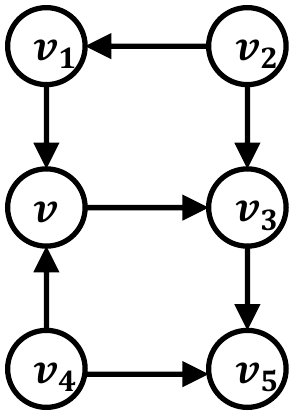}} 
    
    \caption{Examples of 2D local-dependency configurations for vertex \(v\) within its 8-connected neighborhood; for clarity, only the right half is shown.}
    \label{fig:toposortintuit} 
\end{figure*}

The discrete Morse boundary of a critical cell \(\alpha\) consists of all critical cells \(\beta\) connected to \(\alpha\) by an \textit{odd number} of \(V\)-paths in the matching graph. This is a special case of Forman's discrete Morse boundary operator \cite{forman1998} under $\mathbb{Z}_2$ coefficients. A direct computation of this boundary is prohibitive as the number of \(V\)-paths may grow exponentially with the size of the cubical complex, making exhaustive enumeration infeasible \cite{konstantin2013}.

Exploiting the directed acyclic graph (DAG) structure of matching graphs, D{\l}otko and Wagner introduced a sequential method~\cite{hw2012-phcpu} that computes Morse boundaries through two stages: (i) topological sorting of the graph, followed by (ii) \(V\)-path counting between critical cells along the topological order. However, both stages are inherently sequential and pose significant challenges for GPU parallelization. To tackle this, we introduce a parallel topological sorting algorithm and a parallel \(V\)-path parity computation algorithm.

\subsubsection{Parallel Topological Sorting}
\label{subsubsec:tpsort}
Topological sorting of a DAG seeks a linear ordering of its vertices such that for every directed edge \(u \to v\), vertex \(u\) precedes \(v\) in the ordering. This ordering can be visualized as arranging the vertices along a line so that all edges point from left to right. Every DAG admits at least one such ordering \cite{khan1962,cormen2009introduction}. Topological sorting has broad applications, including dependency resolution and task scheduling.

Kahn's classic topological sorting algorithm~\cite{khan1962} iteratively removes zero-in-degree vertices. Ma et al. introduced a parallel variant based on computing the graph's transitive closure \cite{ma1997}, and Er proposed a parallelization strategy that propagates node values from sources to sinks \cite{Er1983}. Saxena et al. devised a hybrid scheme that traces leaf-to-junction paths on GPUs and junction-to-junction paths on CPUs, whose parallel efficiency degrades rapidly as the number of branching nodes increases \cite{rahul2018}. Svantesson and Eklund's GPU adaptation of Kahn's algorithm is likewise constrained by the number of zero-in-degree nodes \cite{david2016}. These approaches fail to harness modern GPUs' massively parallel capabilities because per-iteration parallelism is bounded by the zero-in-degree frontier and the iteration count is at least the DAG height; in the worst-case of a long, narrow DAG, only one vertex is eligible per iteration, collapsing GPU occupancy while synchronization and memory traffic dominate.

We introduce a parallel topological sorting algorithm (Algorithm~\ref{alg:toposort}) that assigns one GPU thread per vertex. Each thread iteratively performs three steps: \textbf{Step I} identifies the predecessor of its assigned vertex via local dependency analysis, \textbf{Step II} traces backward from that vertex to a source vertex, and \textbf{Step III} assigns topological orders using a histogram-based parallel prefix-sum. As orders are committed, edges are removed, and iterations continue until all vertices are ordered. In this paper, we restrict our discussion to matching graphs (DAGs) derived from Morse matching on 2D/3D image data, where vertices have degree four in 2D (up, left, right, down) and degree six in 3D (adding forward and backward). 

\begin{algorithm}[t]
\caption{Parallel Topological Sorting}
\label{alg:toposort}
\DontPrintSemicolon
\KwIn{A DAG $G=(V,E)$ with bounded degree}
\KwOut{Topological index $\tau(v)$ for each $v\in V$}
Initialize $\tau(v)\leftarrow \text{\textup{\textsc{nil}}}$, $pre(v)\leftarrow \text{\textup{\textsc{nil}}}$ for all $v\in V$\;
Initialize $src(v)\leftarrow v$, $dist(v)\leftarrow 0$ for all $v\in V$\;
Initialize $hist$ as an empty hop count histogram\;
Initialize $U\leftarrow V$ as the set of unordered vertices\;
Initialize $F\leftarrow E$ as the set of remaining edges\;
Initialize $offset\leftarrow 0$\;
\While{$U\neq\emptyset$}{
\StepTcp{Step I: Predecessor identification}
\ForParEach{$v\in U$ with $pre(v)=\text{\textup{\textsc{nil}}}$}{
    Identify $v$'s predecessor $pre(v)$ in $(U,F)$\;
}
\StepTcp{Step II: Tracing to a source vertex}
\Repeat{all $pre(src(v))$ have reached a source or $\text{\textup{\textsc{nil}}}$}{
    \ForParEach{$v\in U$ with $pre(src(v))\notin\{\text{\textup{\textsc{nil}}},src(v)\}$}{
        $u\leftarrow pre(src(v))$\;
        $src(v)\leftarrow src(u)$\;
        $dist(v)\leftarrow dist(v)+1+dist(u)$\;
    }
}
\ForParEach{$v\in U$ with $pre(src(v))=src(v)$}{
    Atomically increment $hist[dist(v)]$\;
    $F\leftarrow F\setminus\{(x,y)\in F:x=v\}$\;
}
\StepTcp{Step III: Topological order assignment}
Compute parallel prefix-sum over $hist$\;
$offset^{'}\leftarrow offset$\;
$offset \leftarrow offset + hist[\text{last nonempty bin}]$\;
\ForParEach{$v\!\in\!U$ with $pre(src(v))=src(v)$}{
    Atomically set
    $\begin{cases}
      \tau(v)\leftarrow hist[dist(v)]-1+offset'\\
      hist[dist(v)]\leftarrow hist[dist(v)]-1
    \end{cases}$\;
    $U\leftarrow U\setminus\{v\}$\;
}
}
\end{algorithm}

\textbf{Step I}: Predecessor identification.
For a vertex \(v\), we identify its predecessor \(pre(v) = u\) within its 8-connectivity (2D) or 26-connectivity (3D) neighborhood such that, once \(u\) is placed in the topological order, all dependencies of \(v\) are satisfied and so \(v\) need only appear after \(u\). \(pre(v)\) is identified by examining \(v\)'s in-degree \(d^-(v)\) and its local dependency configuration:
\begin{enumerate}
    \item $d^-(v)=0$ (Figure~\ref{fig:toposortintuit:a}): \(v\) has no dependencies and is therefore a source vertex with $pre(v) = v$.
    \item $d^-(v)=1$ (Figure~\ref{fig:toposortintuit:b}): \(v\) has a single dependency from $v_1$, so $pre(v) = v_1$.
    \item $d^-(v) > 1$ with identifiable predecessor (Figure~\ref{fig:toposortintuit:c}): in the lower half of the graph, \(v\) depends on \(v_3\) and \(v_4\) with \(v_3 \leftarrow v_5 \leftarrow v_4\); in the upper half, \(v\) depends on \(v_1\) and \(v_3\) with \(v_1 \leftarrow v_2 \leftarrow v_3\). Since \(v_1\) transitively precedes \(v_3\) and \(v_4\), satisfying \(v_1\) satisfies all \(v\)'s dependencies, hence \(pre(v) = v_1\).
    \item $d^-(v) > 1$ without identifiable predecessor: when \(v\)'s dependencies lack transitive relationships (e.g., \(v\)'s dependencies from \(v_1\) and \(v_3\) in Figure~\ref{fig:toposortintuit:d}) or their relation is undetermined (e.g., \(v\)'s dependencies from \(v_1\) and \(v_4\) in Figure~\ref{fig:toposortintuit:e}), the predecessor of \(v\) cannot be identified, and $pre(v) = \text{\textsc{nil}}$.
\end{enumerate}

\begin{figure*}[!t]
\centering
\includegraphics[width=7.1in]{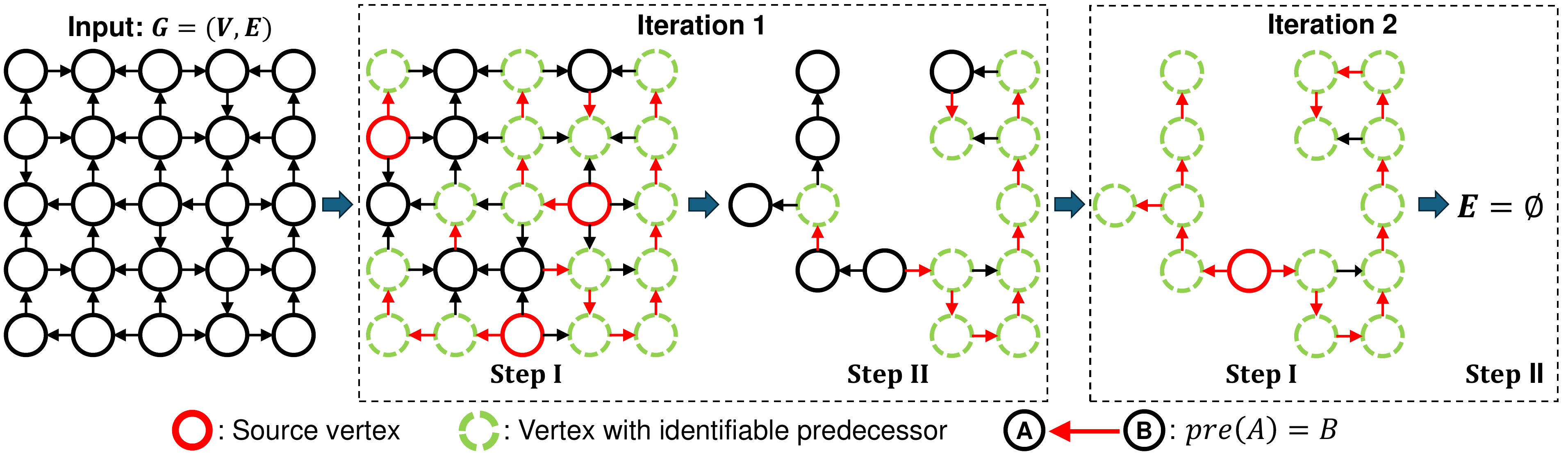}%
\caption{Illustration of Algorithm~\ref{alg:toposort} for parallel topological sorting. \textbf{Step I}: identify predecessors. \textbf{Step II}: trace from green vertices (those with identified predecessors) backward along red arrows (\(pre(\cdot)\)) to red vertices (source vertices), deleting the outgoing edges of each traceable vertex; for clarity, traceable vertices are also removed from the illustration.}
\label{fig:toposortillus}
\end{figure*}

\textbf{Step II}: Tracing to a source vertex.
After identifying the predecessor for vertex \(v\), we trace recursively from \(v\) through \(pre(\cdot)\) until reaching either a source vertex (\( d^-(\cdot) = 0 \)) or a vertex without identifiable predecessor (\(pre(\cdot) =\text{\textsc{nil}}\)). We record in \(src(v)\) the furthest vertex reached in a trace, track in \(dist(v)\) the hop count from \(v\) to \(src(v)\), and maintain a per-iteration histogram \(hist\) whose \(k\)-th bin counts vertices with \(dist(\cdot) = k\). In parallel programming, concurrent updates to shared data by multiple threads can cause data races. We prevent these by using \textit{atomic operations}, which ensure exclusive, uninterrupted access to data by one thread at a time.

\begin{enumerate}
\item Tracing success: If the trace from \(v\) reaches a source vertex, \(v\) is a traceable vertex; we remove all its outgoing edges from the remaining edge set and atomically increment \(hist[dist(v)]\) for hop count \(dist(v)\).

\item Tracing failure: If the trace stops at a vertex with \(pre(\cdot) = \text{\textsc{nil}}\), \(src(v)\) stores the furthest vertex reached in this trace and \(dist(v)\) records the hop count from \(v\) to \(src(v)\). In later iterations, tracing for \(v\) resumes from \(src(v)\), preventing redundant retracing from \(v\).
\end{enumerate}

Figure~\ref{fig:toposortillus} illustrates Algorithm~\ref{alg:toposort} with \textbf{Step III} omitted for clarity. Each vertex of the input DAG \(G\) is handled by a dedicated GPU thread, which independently identifies the vertex's predecessor in \textbf{Step I} and removes its outgoing edges in \textbf{Step II} provided the vertex is traceable. As edges are iteratively removed, remaining vertices’ in-degrees decrease, so vertices with previously unidentifiable predecessors or untraceable paths become processable in subsequent iterations, until no edges remain and all vertices are ordered.

\textbf{Step III}: Topological order assignment.
The inclusive prefix-sum (also known as prefix-scan) transforms an array \(A\) into an array \(B\) of identical length, where each element \( B[i] = \sum_{j=0}^i A[j] \) aggregates the subarray \( A[0\ldots i] \). We adopt the efficient CUDA implementation by Harris et al. \cite{prefix2007} for parallel prefix-sum computation.

A subset of vertices in a DAG forms a \textit{chain} if every pair of vertices is comparable: for any \(u, v\) in the subset, there exists a directed path either from \(u\) to \(v\) or from \(v\) to \(u\). Any valid topological order must preserve the sequence of each chain, though chains need not occupy contiguous positions.

Each trace through the predecessor function \(pre(\cdot)\) to a source vertex forms such a chain where if \(u\) precedes \(v\) in the chain, \(u\) must precede \(v\) in the topological order. We preserve chain sequences by arranging vertices in increasing order of their hop count from their respective source vertices, with vertices sharing identical hop count ordered arbitrarily. This is achieved via inclusive prefix-sum over the per-iteration hop count histogram \(hist\), where topological indices for vertices with hop count \(d\) occupy the range \([hist[d - 1], hist[d])\) for \( d > 0 \), or \([0, hist[0])\) for \( d = 0 \). The global topological order is constructed by appending the vertices ordered in the current iteration to those ordered in prior iterations.

\textbf{Complexity analysis}.
Let \(G=(V,E)\) be a DAG and let \(N = |V|\). Since \(G\) is induced by Morse matchings on 2D/3D image data, every vertex has bounded degree: four in 2D and six in 3D. Hence, \textbf{Step I} performs only local checks in \(\mathcal{O}(1)\) parallel time per iteration.

Let \(r\) be the number of iterations and \(D_t\) be the maximum predecessor-tracing depth among all threads in iteration \(t\). In \textbf{Step II}, synchronized pointer jumping traces a predecessor chain of depth \(D_t\) in \(\lceil \log_2(D_t + 1) \rceil = \Theta(\log(D_t+1))\) parallel rounds. The parallel prefix-sum in \textbf{Step III} has the same asymptotic cost. Since \textbf{Step I} performs constant-time local checks, every iteration costs \(\Theta(1)\) even when \(D_t = 0\). Therefore, iteration \(t\) takes \(\Theta\bigl(1 + \log(D_t + 1)\bigr)\), and the total parallel time is
$$
T = \Theta\!\left(\sum_{t=1}^{r}\bigl(1 + \log(D_t + 1)\bigr)\right).
$$
This bound assumes sufficient thread parallelism and counts synchronization and atomic updates in the standard parallel-time model.

Now let \(m_t\) denote the number of vertices ordered in iteration \(t\) (the \(hist[\text{last nonempty bin}]\) value in Algorithm~\ref{alg:toposort}). Any successful trace of depth \(D_t\) contains \(D_t+1\) vertices, so \(D_t+1 \le m_t\). Since \(\log x \le x\) for \(x \ge 1\), the cost of iteration \(t\) is \(\mathcal{O}(m_t)\). As \(\sum_{t=1}^{r} m_t = N\), we get the general upper bound \(T = \mathcal{O}(N).\) The corresponding general lower bound is \(\Omega(1)\), since any nonempty input requires at least one parallel step.

The complexity, however, depends on how the number of iterations \(r\) and the maximum predecessor-tracing depth \(D_t\) trade off, illustrated by three representative regimes in Figure~\ref{fig:toposortcomplexity}.

\begin{itemize}
    \item \textbf{Large trace depth.}
    Figure~\ref{fig:toposortcomplexity:a} shows a DAG with a single source vertex at the top-left corner, where all vertices trace back to the same source in one iteration. The maximum trace depth is \(D_1=\Theta(N)\), giving \(T=\Theta(\log N).\)

    \item \textbf{Constant depth.}
    Figure~\ref{fig:toposortcomplexity:b} shows a DAG constructed via bipartite orientation: partition vertices into sets \( A = \{ v_{ij} \mid (i+j) \equiv 0 \pmod{2} \} \) and \( B = \{ v_{ij} \mid (i+j) \equiv 1 \pmod{2} \} \), with all edges directed from \( A \) to \( B \). Every vertex is either an immediate source or becomes one after one iteration, so \(r=2\) with \(D_t=\Theta(1)\), giving \(T=\Theta(1).\)

    \item \textbf{Large iteration.}
    Figure~\ref{fig:toposortcomplexity:c} shows a square DAG with a single source vertex at the top-left corner with edge oriented along increasing Manhattan distance from the source (i.e., edges point away from the source). In each iteration, the algorithm removes one boundary layer of the remaining square, so \(r=\Theta(\sqrt{N})\). In iteration \(t\), the longest trace has depth \(D_t = \Theta(\sqrt{N}-t+1)\). Therefore, \(T=\Theta\left(\sum_{t=1}^{\Theta(\sqrt{N})}\left(1+\log(D_t+1)\right)\right)=\Theta(\sqrt{N}\log N).\)
\end{itemize}

These cases show that the running time is governed by the trade-off between \(r\) and \(D_t\): large \(D_t\) can still be efficient when the iteration count is small, while small \(D_t\) may still be costly when many iterations are required.

Algorithm~\ref{alg:toposort} decomposes sequential topological sorting into independent per-vertex tasks executed in parallel. GPU threads operate concurrently within each step, with synchronization via threads barrier after each step.

\begin{figure}[!t]
    \centering
    \captionsetup[subfigure]{font=scriptsize}
    \subfloat[Large trace depth \label{fig:toposortcomplexity:a}]
    {\includegraphics[width=0.13\textwidth]{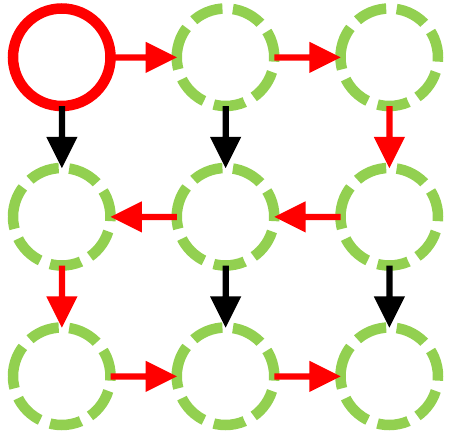}} 
    \hspace{0.5cm}
    \subfloat[Constant depth \label{fig:toposortcomplexity:b}]
    {\includegraphics[width=0.13\textwidth]{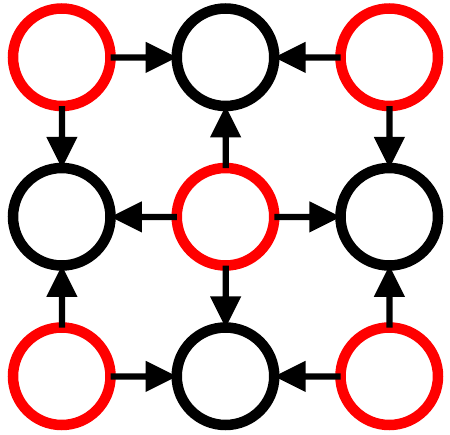}}
    \hspace{0.5cm}
    \subfloat[Large iteration \label{fig:toposortcomplexity:c}]
    {\includegraphics[width=0.13\textwidth]{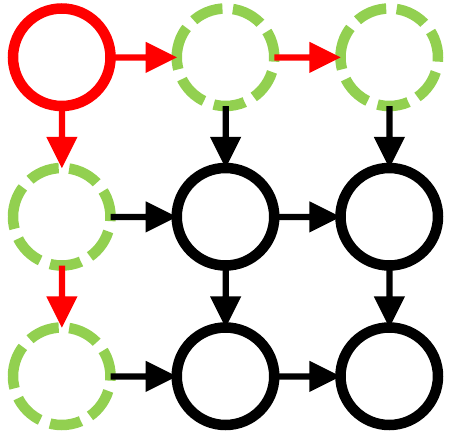}}

    \caption{DAG configurations illustrating three representative regimes for parallel topological sorting: large trace depth, constant depth, and large iteration count; legend as in Figure~\ref{fig:toposortillus}.}
    \label{fig:toposortcomplexity} 
\end{figure}

\subsubsection{Parallel Path Parity Computation}
\label{subsubsec:pathparity}
D{\l}otko and Wagner's Algorithm~4 \cite{hw2012-phcpu} computes Morse boundaries by propagating \(V\)-path counts from each critical cell \(u\) to downstream cells in topological order, adding a critical cell \(v\) to \(u\)'s boundary if an odd number of paths exist from \(u\) to \(v\). For a matching graph \( G = (V, E) \), this approach exhibits a worst-case time complexity of \( \mathcal{O}(c \cdot (|V| + |E|)) \), where \(c\) is the number of critical cells. We reduce path counting to path parity computation and present a parallel algorithm traversing in topological order to track cumulative path parities between critical cells.

We maintain an array \(P\) of length \(|V|\) indexed by the topological order \(\tau(v_i) = i\); each entry \(P[i]\) is itself a subarray of \(T\) 32-bit words, \(P[i][0], \dots, P[i][T-1]\), one word per GPU thread, where \(T\) is  the number of GPU threads. This allows the system to concurrently track up to \(32 \cdot T\) critical cells associated with \(v_i\). Each thread \(q\) is assigned a batch \(\mathcal{B}_q\) of at most 32 critical cells drawn from the critical cell list \(\mathcal{C}\), and the \(k\)-th bit of \(P[i][q]\) stores the cumulative path parity from the \(k\)-th cell of \(\mathcal{B}_q\) to \(v_i\). Encoding each path parity as one bit allows a single word access by a thread to update parities for 32 cell pairs simultaneously. Processing follows topological order \(\tau\), with each thread \(q\) executing:
$$
\forall\, (v_j,v_i)\in E:\; P[i][q] \leftarrow P[i][q] \oplus P[j][q],
$$
where \(\oplus\) denotes bitwise XOR. A critical cell \(v\) lies in the boundary of a critical cell \(u\) (i.e., \(\sigma(u, v) = 1\)) iff the path parity from \(u\) to \(v\) is odd. See Algorithm~\ref{alg:pathcount} for details. 

Assuming \(c \leq 32 \cdot T\), this algorithm demonstrates a time complexity of \(\mathcal{O}(|V| + |E|)\). With \(T \leq 1024\) threads per GPU block and a 1024-thread block covering at most \(12{,}825\) cubical complex cells (CUDA Programming Guide \cite{cuda_programguide}), we have \(c \leq 12{,}825 \leq 32 \cdot 1024 = 32{,}768\); hence the condition \(c \leq 32 \cdot T\) holds in practice.

\textit{GPU global memory coalescing} is a technique that combines memory accesses from a warp (32 GPU threads executing the same instruction simultaneously) into a single transaction to maximize data throughput when threads in a warp access contiguous, aligned memory addresses. We choose \(T\) as a multiple of 32 so each warp operates on a contiguous block of thirty-two 32-bit words, yielding fully coalesced reads and writes.

\begin{algorithm}[t]
\caption{Parallel Path Parity Computation}
\label{alg:pathcount}
\DontPrintSemicolon
\KwIn{A DAG $G=(V,E)$}
\KwIn{Topological order $\tau:V\to\{0,\ldots,|V|-1\}$}
\KwIn{Critical cell list $\mathcal{C}\subseteq V$}
\KwIn{Each of the T GPU threads $q$ is assigned a batch $\mathcal{B}_q\subseteq\mathcal{C}$ of at most $32$ critical cells}
\KwOut{Boundary relations $\sigma(u,v)\!\in\!\{0,1\}$ for $u,\!v\!\in\!\mathcal{C}$}

Initialize $\sigma(\cdot,\cdot)\leftarrow 0$\;
Initialize $P[i][q]\leftarrow 0$ for all $i,q$\;

\ForParEach{GPU thread $q$}{

\tcp{Initialize path parity tracker array}
\ForEach{vertex $v_i$ in topological order}{
    \If{$v_i\in\mathcal{B}_q$}{
        Set bit in $P[i][q]$ corresponding to $v_i$ to $1$\;
    }
}

\tcp{Propagate path parities}
\ForEach{vertex $v_i$ in topological order}{
    \ForEach{edge $(v_j,v_i)\in E$}{
        $P[i][q]\leftarrow P[i][q]\oplus P[j][q]$\;
    }

    \If{$v_i\in\mathcal{C}$}{
        \ForEach{$v_j\in\mathcal{B}_q$}{
            \If{bit in $P[i][q]$ corresponding to $v_j$ is $1$}{
                \If{$\dim(v_j)-\dim(v_i)=1$}{
                    $\sigma(v_j,v_i)\leftarrow 1$\;
                }
            }
        }
    }
}
}
\end{algorithm}

\subsection{Boundary Matrix Reduction on CPU}
\label{subsec:bndmatreduction}
After all chunks have been processed, the complete set of boundary relations is finalized in the host's RAM, thereby avoiding GPU memory capacity limits.

Critical cells are sorted in non-decreasing order by their filtration values using the CUDA Thrust library. Boundary relations \(\sigma(\cdot, \cdot)\) are encoded in a binary boundary matrix over \(
\mathbb{Z}_2\), with columns corresponding to critical cells (ordered by filtration values) and rows indicating their boundary relations. The size of the boundary matrix is quantified as the number of ones in the matrix. We apply the ``twist'' algorithm \cite{twist} implemented in the Persistent Homology Algorithm Toolbox (PHAT) \cite{phat} for boundary matrix reduction. Construction of boundary matrix and extraction of persistence pairs are parallelized with CTPL.

\section{Experiments}
\label{sec:experiments}

\noindent \textbf{Experimental Setup.} Experiments were conducted on a Windows 11 system equipped with an Intel Core i5-13600K CPU (14 cores), 64 GB RAM, and an NVIDIA RTX 4080 GPU with 9,728 CUDA cores and 16 GB GDDR6 memory. The algorithm is implemented in CUDA 12.7 and C++17, with data residing on a 4TB Seagate hard disk drive operating at 5400 RPM. CUDA kernel launch parameters are selected at runtime using CUDA's occupancy-based launch configuration API, \texttt{cudaOccupancyMaxPotentialBlockSize}, which determines the block size that maximizes potential occupancy for the target kernel on the given GPU. The final grid size and dynamically allocated shared memory size are then computed from the input size and the selected block size.

\noindent \textbf{Datasets.}
\begin{itemize}
    \item \textbf{Volumetric Simulation and Imaging}~\cite{scivisdata}: 28 scientific visualization volumes, of which we report a representative subset of 22 in the main paper; the full evaluation is provided in Table~S.I of the supplementary material.
    \item \textbf{Virtual Imaging Clinical Trials for Regulatory Evaluation (VICTRE)} ~\cite{VICTRE01}: Simulated breast phantoms for four tissue scenarios (Dense, Fatty, Heterogeneous, Scattered), with 5 MRI volumes synthesized per scenario.
    \item \textbf{Gaussian Random Fields (GRF)}: \(512^3\)-resolution fields generated with varying spectral index \(\alpha\) (10 volumes per \(\alpha\)). Smaller \(\alpha\) values produce finer-scale fluctuations.
\end{itemize}

\noindent \textbf{Baselines.}
\begin{itemize}
    \item \textbf{DIPHA} \cite{ulrich2014}: A parallel algorithm using the Message Passing Interface (MPI) for persistent homology computation. It exhibits a substantial memory footprint. Timings are reported for parallel execution on a single machine.
    \item \textbf{GUDHI} \cite{gudhi}: A comprehensive topological data analysis library prioritizing usability and seamless integration into broader workflows.
    \item \textbf{CRipser} \cite{kaji2020}: A widely recognized, specialized adaptation of Ripser \cite{ripser} for cubical complexes. For comparison, we used its C++ implementation, which ran faster than the Python version.
    \item \textbf{Cubicle} \cite{chunking}: A multi-core parallel algorithm for persistent homology computation, employing an external-memory approach to handle large datasets. It lacks support for floating-point inputs.
\end{itemize}

\setlength{\tabcolsep}{5.0pt}
\begin{table*}[!htb]
\centering
\begin{threeparttable}
\caption{End-to-end runtime for persistent homology computation: TopoGPU vs.\ state-of-the-art baselines.}

\begin{tabular}{@{}cccccccccc@{}}
\hline
\multirow{2}{*}{Dataset} 
& \multicolumn{1}{c}{Avg. \#Persist.} 
& \multicolumn{1}{c}{\textbf{TopoGPU}} 
& \multicolumn{1}{c}{Cubicle} 
& \multicolumn{1}{c}{CRipser} 
& \multicolumn{1}{c}{DIPHA} 
& \multicolumn{1}{c}{GUDHI} 
& \multicolumn{1}{c}{Topo/CRipser} 
& \multicolumn{1}{c}{Topo/best} 
& \multicolumn{1}{c}{TopoGPU} \\
& Pairs 
& (secs) 
& (secs) 
& \multicolumn{1}{c}{(secs)} 
& (secs) 
& (secs) 
& \multicolumn{1}{c}{Speedup} 
& \multicolumn{1}{c}{Speedup\tnotea} 
& \multicolumn{1}{c}{MVox/s\tnoteb} \\ \hline

Silicium & 479 & 0.02\text{\scriptsize$\pm$0.00} & 0.62\text{\scriptsize$\pm$0.00} & 0.11\text{\scriptsize$\pm$0.00} & 0.66\text{\scriptsize$\pm$0.00} & 0.56\text{\scriptsize$\pm$0.03} & 5.50 & 5.50 & 5.66\\
Fuel & 126 & 0.03\text{\scriptsize$\pm$0.00} & 0.65\text{\scriptsize$\pm$0.04} & 0.28\text{\scriptsize$\pm$0.00} & 1.43\text{\scriptsize$\pm$0.03} & 1.33\text{\scriptsize$\pm$0.00} & 9.33 & 9.33 & 8.74\\
Neghip & 346 & 0.03\text{\scriptsize$\pm$0.00} & 0.76\text{\scriptsize$\pm$0.02} & 0.31\text{\scriptsize$\pm$0.00} & 1.48\text{\scriptsize$\pm$0.02} & 1.36\text{\scriptsize$\pm$0.03} & 10.33 & 10.33 & 8.74\\
Tooth & 204769 & 0.59\text{\scriptsize$\pm$0.01} & 2.29\text{\scriptsize$\pm$0.03} & 3.80\text{\scriptsize$\pm$0.11} & 9.90\text{\scriptsize$\pm$0.10} & 11.19\text{\scriptsize$\pm$0.17} & 6.44 & 3.88 & 2.64\\
Hydrogen & 7 & 0.12\text{\scriptsize$\pm$0.00} & 0.89\text{\scriptsize$\pm$0.03} & 3.67\text{\scriptsize$\pm$0.07} & 12.04\text{\scriptsize$\pm$0.31} & 12.16\text{\scriptsize$\pm$0.04} & 30.58 & 7.42 & 17.48\\
Shockwave & 26 & 0.11\text{\scriptsize$\pm$0.00} & 0.78\text{\scriptsize$\pm$0.01} & 4.25\text{\scriptsize$\pm$0.09} & 11.88\text{\scriptsize$\pm$0.46} & 11.09\text{\scriptsize$\pm$0.04} & 38.64 & 7.09 & 19.06\\
Lobster & 140275 & 0.86\text{\scriptsize$\pm$0.00} & 4.56\text{\scriptsize$\pm$0.02} & 13.02\text{\scriptsize$\pm$0.18} & 35.42\text{\scriptsize$\pm$0.08} & 37.37\text{\scriptsize$\pm$0.44} & 15.14 & 5.30 & 6.35\\
Ventricles & 1727530 & 4.23\text{\scriptsize$\pm$0.04} & 14.50\text{\scriptsize$\pm$0.24} & 30.81\text{\scriptsize$\pm$0.55} & 54.56\text{\scriptsize$\pm$0.77} & 76.06\text{\scriptsize$\pm$0.11} & 7.28 & 3.43 & 1.92\\
Engine & 242894 & 1.31\text{\scriptsize$\pm$0.01} & 5.36\text{\scriptsize$\pm$0.05} & 24.40\text{\scriptsize$\pm$0.14} & 52.86\text{\scriptsize$\pm$0.69} & 60.59\text{\scriptsize$\pm$0.48} & 18.63 & 4.09 & 6.40\\
Bonsai & 109008 & 0.70\text{\scriptsize$\pm$0.01} & 3.49\text{\scriptsize$\pm$0.04} & 76.51\text{\scriptsize$\pm$5.25} & 99.45\text{\scriptsize$\pm$0.62} & 115.3\text{\scriptsize$\pm$1.08} & 109.30 & 4.99 & 23.97\\
Aneurysm & 34494 & 0.51\text{\scriptsize$\pm$0.01} & 1.87\text{\scriptsize$\pm$0.02} & 54.68\text{\scriptsize$\pm$1.44} & 96.05\text{\scriptsize$\pm$0.36} & 111.7\text{\scriptsize$\pm$1.53} & 107.22 & 3.67 & 32.90\\
Foot & 749676 & 1.77\text{\scriptsize$\pm$0.01} & 7.25\text{\scriptsize$\pm$0.05} & 65.80\text{\scriptsize$\pm$0.29} & 102.3\text{\scriptsize$\pm$1.60} & 123.4\text{\scriptsize$\pm$1.25} & 37.18 & 4.10 & 9.48\\
Turbulence & 376226 & 1.15\text{\scriptsize$\pm$0.00} & N/A\tnotec & 48.73\text{\scriptsize$\pm$0.71} & 129.0\text{\scriptsize$\pm$3.16} & 96.43\text{\scriptsize$\pm$0.58} & 42.37 & 42.37 & 14.60\\
Skull & 1502436 & 4.50\text{\scriptsize$\pm$0.05} & 15.60\text{\scriptsize$\pm$0.04} & 85.19\text{\scriptsize$\pm$4.77} & 114.5\text{\scriptsize$\pm$1.51} & 147.9\text{\scriptsize$\pm$0.62} & 18.93 & 3.47 & 3.73\\
Angio & 5761984 & 13.99\text{\scriptsize$\pm$0.06} & 41.48\text{\scriptsize$\pm$0.79} & 120.7\text{\scriptsize$\pm$5.48} & 165.7\text{\scriptsize$\pm$2.76} & 279.4\text{\scriptsize$\pm$1.65} & 8.63 & 2.96 & 1.71\\
Heptane & 27844 & 0.88\text{\scriptsize$\pm$0.00} & 2.78\text{\scriptsize$\pm$0.04} & 110.7\text{\scriptsize$\pm$1.73} & 168.0\text{\scriptsize$\pm$3.25} & 190.5\text{\scriptsize$\pm$1.10} & 125.80 & 3.16 & 31.30\\

\noalign{\smallskip}
\hline
\noalign{\smallskip}

Vertebra & 2782360 & 10.57\text{\scriptsize$\pm$0.14} & 38.79\text{\scriptsize$\pm$0.34} & 503.4\text{\scriptsize$\pm$17.96} & 1135.8\text{\scriptsize$\pm$6.50} & 1342.8\text{\scriptsize$\pm$15.61} & 47.63 & 3.67 & 12.70\\
Zeiss & 423233 & 12.87\text{\scriptsize$\pm$0.15} & 48.02\text{\scriptsize$\pm$0.58} & 2422.9\text{\scriptsize$\pm$214.8} & 4166.6\text{\scriptsize$\pm$195.1} & \xmark\tnoted & 188.26 & 3.73 & 24.43\\
Prone & 14038282 & 36.70\text{\scriptsize$\pm$0.31} & 163.2\text{\scriptsize$\pm$4.06} & 566.6\text{\scriptsize$\pm$19.12} & 973.2\text{\scriptsize$\pm$18.21} & 1300.1\text{\scriptsize$\pm$2.28} & 15.44 & 4.45 & 3.31\\
Marmoset & 43599571 & 251.1\text{\scriptsize$\pm$2.00} & 834.8\text{\scriptsize$\pm$10.50} & \xmark & \xmark & \xmark & - & 3.33 & 1.31\\
Pawpawsaurus & 90157269 & 449.9\text{\scriptsize$\pm$5.69} & 1395.7\text{\scriptsize$\pm$20.95} & \xmark & \xmark & \xmark & - & 3.10 & 1.50\\
Chameleon & 52031009 & 238.2\text{\scriptsize$\pm$7.43} & 807.7\text{\scriptsize$\pm$7.82} & \xmark & \xmark & \xmark & - & 3.39 & 4.75\\

\noalign{\smallskip}
\hline
\noalign{\smallskip}

Dense & 136899 & 4.79\text{\scriptsize$\pm$0.22} & N/A & 208.2\text{\scriptsize$\pm$6.26} & 203.0\text{\scriptsize$\pm$5.16} & 138.5\text{\scriptsize$\pm$2.28} & 43.47 & 28.91 & 4.34\\
Fatty & 516830 & 7.60\text{\scriptsize$\pm$0.13} & N/A & 546.79\text{\scriptsize$\pm$11.84} & 762.5\text{\scriptsize$\pm$7.16} & 906.2\text{\scriptsize$\pm$30.79} & 71.95 & 71.95 & 14.74\\
Hetero & 231627 & 6.76\text{\scriptsize$\pm$0.25} & N/A & 379.7\text{\scriptsize$\pm$13.80} & 287.3\text{\scriptsize$\pm$2.50} & 363.4\text{\scriptsize$\pm$26.86} & 56.17 & 42.50 & 5.70\\
Scattered & 340780 & 10.47\text{\scriptsize$\pm$0.32} & N/A & 825.8\text{\scriptsize$\pm$92.05} & 486.0\text{\scriptsize$\pm$5.25} & 779.8\text{\scriptsize$\pm$16.90} & 78.87 & 46.42 & 7.11\\

\noalign{\smallskip}
\hline
\noalign{\smallskip}

$\text{GRF}^{\alpha=1.0}$ & 42548195 & 160.0\text{\scriptsize$\pm$2.48} & 525.0\text{\scriptsize$\pm$10.95} & 1295.7\text{\scriptsize$\pm$47.19} & 2233.7\text{\scriptsize$\pm$198.0} & 2353.7\text{\scriptsize$\pm$62.40} & 8.10 & 3.28 & 0.84\\
$\text{GRF}^{\alpha=2.5}$ & 29496809 & 94.62\text{\scriptsize$\pm$0.66} & 359.0\text{\scriptsize$\pm$3.71} & 1029.2\text{\scriptsize$\pm$17.13} & 1420.0\text{\scriptsize$\pm$24.0} & 2079.2\text{\scriptsize$\pm$30.20} & 10.88 & 3.79 & 1.42\\
$\text{GRF}^{\alpha=5.0}$ & 1090669 & 7.22\text{\scriptsize$\pm$0.03} & 27.70\text{\scriptsize$\pm$2.98} & 919.5\text{\scriptsize$\pm$18.46} & 1212.0\text{\scriptsize$\pm$6.09} & 1294.4\text{\scriptsize$\pm$106.0} & 127.35 & 3.84 & 18.59\\
$\text{GRF}^{\alpha=7.5}$ & 27 & 4.87\text{\scriptsize$\pm$0.00} & 11.51\text{\scriptsize$\pm$4.03} & 964.3\text{\scriptsize$\pm$13.64} & 1170.7\text{\scriptsize$\pm$16.41} & 1135.9\text{\scriptsize$\pm$23.93} & 198.01 & 2.36 & 27.56\\

\noalign{\smallskip}
\hline
\noalign{\smallskip}

Average  & \textbf{-----} & \textbf{---} & \textbf{---} & \textbf{---} & \textbf{---} & \textbf{---} & 53.24 & 11.53 & 10.77\\
\noalign{\smallskip}

\hline
\end{tabular}

\begin{tablenotes}\footnotesize
\setlength{\premulticols}{0pt}
\setlength{\multicolsep}{0pt}
\setlength{\postmulticols}{0pt}
\begin{multicols}{2}
\item[a] Topo/best Speedup: TopoGPU's speedup over the fastest baseline.
\item[b] TopoGPU MVox/s: throughput in millions of voxels per second.
\item[c] N/A: Cubicle does not support floating-point inputs.
\item[d] \xmark: DIPHA, GUDHI, and CRipser failed on large or complex inputs.
\end{multicols}
\end{tablenotes}

\label{tab:wholepipe}
\end{threeparttable}
\end{table*}

We evaluated TopoGPU from two perspectives: (i) end-to-end performance of the full persistent homology computation pipeline, from data loading through boundary matrix reduction, and (ii) a stage-wise runtime breakdown of TopoGPU reporting per-stage time percentages and peak host RAM usage for each dataset. All tables report standard deviations over five runs. TopoGPU produces results identical to those of tcripser (CRipser with the T-construction).

\subsection{Performance Comparisons in End-to-End Persistent Homology Computation}
Table~\ref{tab:wholepipe} reports end-to-end runtimes. TopoGPU outperforms all baselines across every dataset. Against CRipser it achieves up to \(198.01 \times\) speedup (average \(53.24 \times\)). Relative to the best-performing baseline, it reaches up to \(71.95 \times\) speedup (average \(11.53 \times\)). We report TopoGPU's speedup over CRipser as it is arguably the most widely used package for persistent homology computation on cubical complexes; speedups over each baseline are in Table~S.II of the supplementary material.

DIPHA, GUDHI, and CRipser crashed on large inputs due to their non-streaming designs, and they failed on complex cases producing oversized boundary matrices, which easily exceeded available RAM. On our datasets, the boundary matrix produced by TopoGPU was, on average, \(8.99\times\) larger than that of the Morse complex without partitioning. Nevertheless, TopoGPU still provides a substantial reduction over the cubical complex: the boundary matrix of the cubical complex was, on average, \(321.46\times\) larger than that of TopoGPU. Detailed comparisons are in Table~S.III of the supplementary material.

A key metric to gauge computation efficiency is \textbf{voxel throughput}, defined as the number of voxels processed per second and measured in millions (\(10^6\)) of voxels per second (MVox/s). TopoGPU achieved an average throughput of 10.77 MVox/s, peaking at 32.90 MVox/s on the Aneurysm volume.

\setlength{\tabcolsep}{5.0pt}
\begin{table*}[!htb]
\caption{Stage-wise runtime breakdown and peak host RAM and device memory usage for TopoGPU.}
\centering
\begin{threeparttable}
\begin{tabular}{cccccccccccc}
\hline
\multirow{2}{*}{Dataset} & \multirow{2}{*}{Size} & \multirow{2}{*}{\shortstack{Data\\Type}} & \multirow{2}{*}{\shortstack{BoundMat\\Size\tnote{a}}} & \multirow{2}{*}{Dev$\leftrightarrow$Host} & \multirow{2}{*}{\shortstack{Morse\\Matching}} & \multirow{2}{*}{\shortstack{Topo\\Sort}} & \multirow{2}{*}{\shortstack{Path\\Parity}} & \multirow{2}{*}{\shortstack{Matrix\\Reduction}} & \multirow{2}{*}{I/O} & \multicolumn{1}{c}{RAM\tnote{b}} & \multicolumn{1}{c}{VRAM\tnote{c}} \\
& & & & & & & & & & (GB) & (GB) \\ \hline
Silicium & $98\!\times\!34\!\times\!34$ & uint8 & $3.24 \cdot 10^{4}$ & 0.00\% & 1.58\% & 2.40\% & 22.20\% & 44.72\% & 10.17\% & 0.19 & 0.34\\
Fuel & $64\!\times\!64\!\times\!64$ & uint8 & $5.47 \cdot 10^{4}$ & 0.00\% & 1.48\% & 3.12\% & 47.76\% & 28.02\% & 14.62\% & 0.19 & 0.31\\
Neghip & $64\!\times\!64\!\times\!64$ & uint8 & $9.09 \cdot 10^{4}$ & 0.00\% & 1.39\% & 3.57\% & 48.08\% & 26.78\% & 10.18\% & 0.20 & 0.31\\
Tooth & $103\!\times\!94\!\times\!161$ & uint8 & $2.80 \cdot 10^{6}$ & 0.19\% & 0.63\% & 1.25\% & 15.38\% & 77.85\% & 3.54\% & 0.42 & 0.53\\
Hydrogen & $128\!\times\!128\!\times\!128$ & uint8 & $2.72 \cdot 10^{5}$ & 0.02\% & 2.64\% & 6.94\% & 55.33\% & 19.82\% & 14.05\% & 0.28 & 0.91\\
Shockwave & $64\!\times\!64\!\times\!512$ & uint8 & $3.98 \cdot 10^{4}$ & 0.01\% & 1.93\% & 5.69\% & 78.76\% & 5.69\% & 7.62\% & 0.22 & 0.31\\
Lobster & $301\!\times\!324\!\times\!56$ & uint8 & $3.61 \cdot 10^{6}$ & 0.28\% & 0.82\% & 2.22\% & 15.36\% & 74.37\% & 2.95\% & 0.62 & 1.74\\
Ventricles & $256\!\times\!256\!\times\!124$ & uint8 & $1.82 \cdot 10^{7}$ & 1.31\% & 0.54\% & 1.35\% & 10.11\% & 84.31\% & 1.01\% & 1.62 & 0.91\\
Engine & $256\!\times\!256\!\times\!128$ & uint8 & $5.55 \cdot 10^{6}$ & 0.41\% & 0.71\% & 2.50\% & 16.83\% & 71.51\% & 3.96\% & 0.65 & 0.91\\
Bonsai & $256\!\times\!256\!\times\!256$ & uint8 & $2.84 \cdot 10^{6}$ & 0.22\% & 4.72\% & 7.49\% & 48.51\% & 26.44\% & 10.12\% & 0.50 & 0.91\\
Aneurysm & $256\!\times\!256\!\times\!256$ & uint8 & $1.02 \cdot 10^{6}$ & 0.09\% & 6.36\% & 9.18\% & 60.67\% & 12.12\% & 9.22\% & 0.39 & 0.91\\
Foot & $256\!\times\!256\!\times\!256$ & uint8 & $8.11 \cdot 10^{6}$ & 0.61\% & 0.96\% & 3.04\% & 21.24\% & 67.21\% & 5.15\% & 0.96 & 0.91\\
Turbulence & $256\!\times\!256\!\times\!256$ & float32 & $5.31 \cdot 10^{6}$ & 0.40\% & 1.06\% & 5.40\% & 31.99\% & 45.39\% & 12.25\% & 0.78 & 0.92\\
Skull & $256\!\times\!256\!\times\!256$ & uint8 & $2.01 \cdot 10^{7}$ & 1.38\% & 0.46\% & 1.67\% & 18.50\% & 74.17\% & 2.23\% & 1.74 & 0.91\\
Angio & $416\!\times\!512\!\times\!112$ & uint16 & $5.05 \cdot 10^{7}$ & 3.61\% & 0.32\% & 0.98\% & 10.61\% & 83.61\% & 0.60\% & 4.10 & 1.91\\
Heptane & $302\!\times\!302\!\times\!302$ & uint8 & $1.64 \cdot 10^{6}$ & 0.19\% & 6.92\% & 9.37\% & 57.88\% & 14.52\% & 10.61\% & 0.71 & 1.74\\
\noalign{\smallskip}\hline\noalign{\smallskip}
Vertebra & $512\!\times\!512\!\times\!512$ & uint16 & $4.14 \cdot 10^{7}$ & 3.04\% & 1.58\% & 4.32\% & 32.92\% & 49.29\% & 4.79\% & 4.50 & 2.09\\
Zeiss & $680\!\times\!680\!\times\!680$ & uint8 & $4.62 \cdot 10^{7}$ & 3.51\% & 4.61\% & 8.39\% & 39.10\% & 37.91\% & 5.25\% & 6.29 & 2.05\\
Prone & $512\!\times\!512\!\times\!463$ & uint16 & $1.29 \cdot 10^{8}$ & 9.32\% & 0.30\% & 1.22\% & 10.74\% & 75.62\% & 1.22\% & 10.16 & 2.09\\
Marmoset & $1024\!\times\!1024\!\times\!314$ & uint8 & $5.40 \cdot 10^{8}$ & 35.02\% & 0.28\% & 1.12\% & 6.56\% & 55.78\% & 0.24\% & 38.94 & 2.06\\
Pawpawsaurus & $958\!\times\!646\!\times\!1088$ & uint16 & $8.31 \cdot 10^{8}$ & 27.36\% & 0.29\% & 1.12\% & 6.39\% & 61.59\% & 0.30\% & 54.04 & 1.96\\
Chameleon & $1024\!\times\!1024\!\times\!1080$ & uint16 & $5.82 \cdot 10^{8}$ & 26.75\% & 0.29\% & 1.42\% & 9.38\% & 61.19\% & 0.83\% & 42.28 & 4.17\\
\noalign{\smallskip}\hline\noalign{\smallskip}
Dense & $287\!\times\!359\!\times\!202$ & float32 & $8.53 \cdot 10^{6}$ & 0.80\% & 1.03\% & 2.46\% & 16.41\% & 74.97\% & 3.05\% & 1.15 & 1.56\\
Fatty & $440\!\times\!518\!\times\!488$ & float32 & $2.18 \cdot 10^{7}$ & 1.82\% & 2.74\% & 6.58\% & 33.47\% & 47.78\% & 7.11\% & 2.57 & 2.00\\
Hetero & $359\!\times\!359\!\times\!299$ & float32 & $1.32 \cdot 10^{7}$ & 1.07\% & 0.99\% & 2.76\% & 15.43\% & 73.84\% & 3.11\% & 1.62 & 1.56\\
Scattered & $434\!\times\!446\!\times\!384$ & float32 & $2.14 \cdot 10^{7}$ & 1.77\% & 1.00\% & 3.70\% & 18.33\% & 70.74\% & 3.85\% & 2.55 & 1.73\\
\noalign{\smallskip}\hline\noalign{\smallskip}
$\text{GRF}^{\alpha=1.0}$ & $512\!\times\!512\!\times\!512$ & int32 & $3.38 \cdot 10^{8}$ & 24.19\% & 0.11\% & 0.39\% & 8.50\% & 66.49\% & 0.31\% & 24.54 & 2.16\\
$\text{GRF}^{\alpha=2.5}$ & $512\!\times\!512\!\times\!512$ & int32 & $2.48 \cdot 10^{8}$ & 17.88\% & 0.33\% & 1.19\% & 9.70\% & 69.92\% & 0.53\% & 18.40 & 2.16\\
$\text{GRF}^{\alpha=5.0}$ & $512\!\times\!512\!\times\!512$ & int32 & $2.64 \cdot 10^{7}$ & 2.59\% & 1.50\% & 8.06\% & 33.27\% & 46.04\% & 7.26\% & 3.18 & 2.16\\
$\text{GRF}^{\alpha=7.5}$ & $512\!\times\!512\!\times\!512$ & int32 & $4.46 \cdot 10^{5}$ & 0.59\% & 7.62\% & 12.38\% & 39.31\% & 23.91\% & 14.19\% & 1.15 & 2.16\\
\noalign{\smallskip}\hline\noalign{\smallskip}
Average  & \textbf{---} & \textbf{---} & \textbf{---} & 5.48\% & 1.84\% & 4.04\% & 27.96\%  & 52.39\% & 5.68\% & \textbf{---} & \textbf{---}\\
\hline
\end{tabular}
\begin{tablenotes}\footnotesize
\item[a] BoundMat Size: the number of ones over \( \mathbb{Z}_2\) in the discrete Morse boundary matrix.
\item[b] RAM: peak host (CPU) RAM used by TopoGPU.
\item[c] VRAM: peak device (GPU) memory used by TopoGPU.
\end{tablenotes}
\label{tab:breakdown}
\end{threeparttable}
\end{table*}

\subsection{Runtime Breakdown}
Table~\ref{tab:breakdown} reports TopoGPU's stage-wise runtime breakdown and peak RAM and VRAM usage. On average, boundary matrix reduction dominates runtime at \(52.39\%\), followed by path-parity computation at \(27.96\%\), I/O at \(5.68\%\), and device-host transfers at \(5.48\%\). Topological sorting and Morse matching are minor.

Small and medium volumes are often path parity heavy (e.g., Shockwave \(78.76\%\), Bonsai \(48.51\%\)), whereas large and complex volumes are reduction heavy (e.g., Pawpawsaurus \(61.59\%\), Chameleon \(61.19\%\)). Device-host transfer is dominated by sending boundary relations from device to host and varies widely--from negligible (e.g., Silicium \(0.00\%\)) to substantial (e.g., Marmoset \(35.02\%\)). Peak RAM usage scales with boundary matrix size, ranging from 0.19 GB on the smallest volume to 54.04 GB on Pawpawsaurus. TopoGPU ranks the second-lowest in peak RAM among the baselines, surpassed only by Cubicle which uses less RAM by offloading boundary relations to external storage.

Figure~\ref{fig:proportion} presents a stacked-bar runtime breakdown of TopoGPU for the first 16 volumes; colored bars indicate per-stage shares, and the overlaid black curve shows total runtime per volume. From the column reporting the number of persistence pairs in Table~\ref{tab:wholepipe}, path parity (orange bars) dominates runtime when the pair count is small, whereas boundary matrix reduction (green bars) dominates when it is large.  The remainder (yellow bars) comprises initialization, array allocation, and resource release.

We also observe a positive correlation between total runtime and the share of boundary matrix reduction. We quantify this trend by computing the Pearson correlation \cite{pearson1895} between \(\log_{10}\) total runtime and the percentage of time spent in boundary matrix reduction, obtaining \(r = 0.64\) with \(p = 0.0076\). We further compute the Spearman rank correlation \cite{spearman04} between these two quantities, obtaining \(\rho = 0.62\) with \(p = 0.010\). These results indicate a statistically significant positive association between longer TopoGPU runtimes and a greater matrix reduction share, suggesting that matrix reduction is the primary bottleneck in our GPU persistent homology computation pipeline.

\begin{figure*}[!htb]
\centering
\includegraphics[width=7.1in]{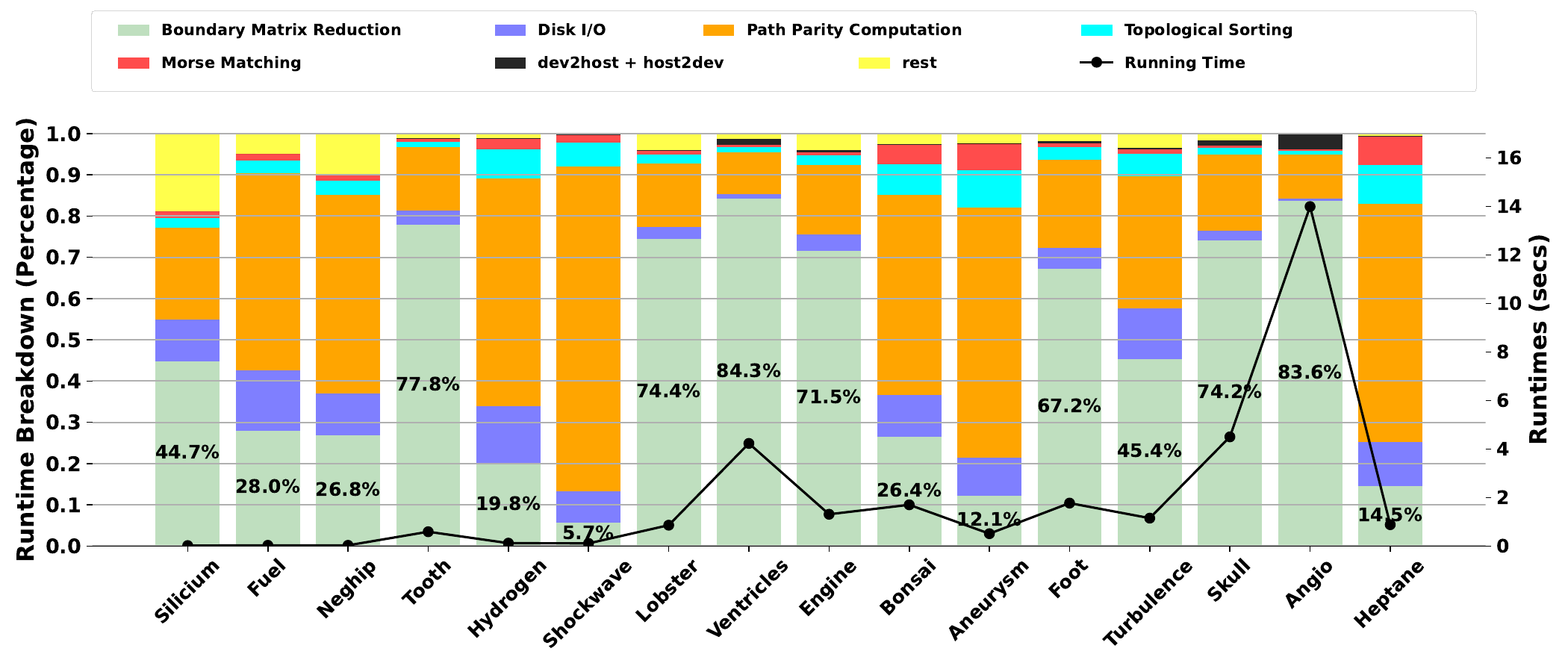}%
\caption{Per-stage runtime breakdown of TopoGPU. Stacked bars show stage-wise percentages (left y-axis); the black curve shows total runtime of the first 16 volumes (right y-axis). Percentages inside the green bars indicate the share of boundary matrix reduction.}
\label{fig:proportion}
\end{figure*}

\subsection{Comparative Analysis of Topological Sorting}
TopoSort is a parallel topological sorting algorithm for the bounded-degree matching graphs arising in our setting, and its applicability extends beyond persistent homology computation to DAGs with similar local structure. This subsection evaluates TopoSort as a standalone algorithm: the inputs are matching graphs, and the output is a topological ordering. We compare TopoSort with two baselines:

\begin{itemize}
    \item \textbf{CPUKahn}: the classical Kahn's algorithm \cite{khan1962}, which iteratively identifies zero-in-degree vertices, appends them to the topological order, and removes their outgoing edges.

    \item \textbf{GPUKahn}: a na\"ive GPU parallelization of Kahn's algorithm, which concurrently identifies and removes zero-in-degree vertices. Its parallelism is therefore limited by the size of the zero-in-degree frontier in each iteration.
\end{itemize}

Figure~\ref{fig:toposorttiming} shows runtime curves for CPUKahn (black), GPUKahn (blue), and TopoSort (red), along with the TopoSort-over-GPUKahn speedup curve (orange), across 16 volumes ordered by increasing matching graph vertex count (annotated above each bar). TopoSort consistently outperforms CPUKahn and GPUKahn, achieving an average of 46.13$\times$ and up to 54.81$\times$ speedup over CPUKahn, and 1.61$\times$--9.80$\times$ speedups over GPUKahn. While runtimes of CPUKahn and GPUKahn increase with vertex count, TopoSort exhibits superior scalability, with markedly flatter runtime growth.

The figure also includes stacked bar charts showing the percentage of vertices topologically sorted while \(d^-(\cdot) > 0\) (green) versus \(d^-(\cdot) = 0\) (red). On average, 89.0\% of vertices are sorted while still having nonzero in-degree. This highlights the benefit of TopoSort over source-frontier-based methods: instead of processing only zero-in-degree vertices, TopoSort can order many vertices whose dependencies are already resolved through predecessor tracing, thereby exposing substantially more parallelism. We further report the empirical iteration counts of TopoSort in Table~S.V of the supplementary material, showing that despite the large variation in graph size, the number of iterations remains small and stable across all volumes, ranging from 11 to 17 with an average of 13.25.

\begin{figure*}[!htb]
\centering
\includegraphics[width=7.0in]{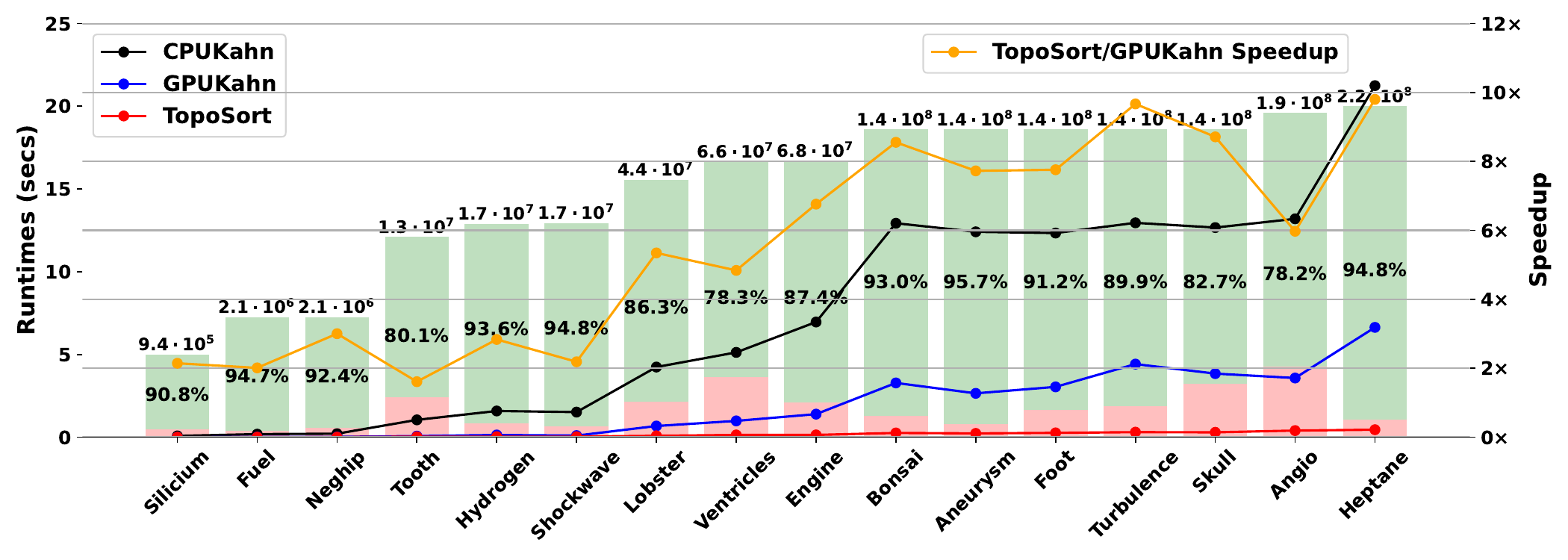}%
\caption{
Runtime curves for CPUKahn, GPUKahn, and TopoSort (left y-axis) with TopoSort-over-GPUKahn speedup curve (right y-axis). Stacked bar charts show the percentage of vertices topologically sorted while \(d^-(\cdot) > 0\) (green; percentage displayed) versus when \(d^-(\cdot) = 0\) (red). Volumes are ordered by increasing vertex count in their matching graphs, with vertex count annotated above each bar whose height corresponds to the count.}
\label{fig:toposorttiming}
\end{figure*}

\subsection{Comparison with GPU-Based 2D Cubical Persistence}
To our knowledge, CuMPerLay~\cite{cmp} is the only other GPU implementation of cubical persistence, and it is restricted to 2D images. We therefore compare against it in its native 2D setting on the two datasets used in its evaluation: the ISIC 2018 test set~\cite{ISIC1,ISIC2} (1{,}000 images, resized to $512\times512$) and the Eye Disease (Glaucoma) dataset~\cite{Glaucoma} (5{,}335 images at their original $1690\times2004$ resolution), all converted to grayscale.

CuMPerLay derives its throughput from batching. Its input is a 4D tensor designed for batched multiparameter persistence within deep networks, which is internally flattened into a stack of 2D images. Each image is handled concurrently by its own GPU worker, while the pairing loop within an image is sequential. We therefore fill both the batch and channel dimensions with distinct images, sweep their sizes to fully exploit the GPU, and report its best configuration. TopoGPU is likewise run with concurrent processes tuned for best throughput. All experiments use an NVIDIA RTX 4080. Table~\ref{tab:cmp} reports end-to-end runtimes: TopoGPU is $3.3\times$ faster on ISIC 2018 and $24.6\times$ faster on the higher-resolution Glaucoma dataset, where per-image parallelism increasingly outweighs batch-level concurrency.

\begin{table}[!t]
\caption{End-to-end runtime versus CuMPerLay on 2D image datasets.}
\centering
\begin{tabular}{ccccc}
\hline
\multirow{2}{*}{Dataset} & \multirow{2}{*}{\shortstack{Resolution\\/ \#Images}} & \multicolumn{2}{c}{Avg. time per image (ms)} & \multirow{2}{*}{Speedup} \\
 & & CuMPerLay & TopoGPU & \\ \hline
ISIC 2018 & $512\times512$ / 1{,}000 & 21.30 & 6.43 & 3.31$\times$ \\
Glaucoma  & $1690\times2004$ / 5{,}335 & 1128.44 & 45.87 & 24.60$\times$ \\ \hline
\end{tabular}
\label{tab:cmp}
\end{table}

\subsection{Deep Learning Pipeline Integration}
To demonstrate TopoGPU's practicality, we integrate it into a topology-preserving deep learning segmentation workflow of Hu et al.\cite{hu2019}. Their framework enforces topological fidelity of the segmentation results by minimizing the \textit{Wasserstein distance} between the persistence diagrams of the network's predicted likelihood map and the ground truth.

\begin{figure}[!htb]
\centering
\includegraphics[width=3.5in]{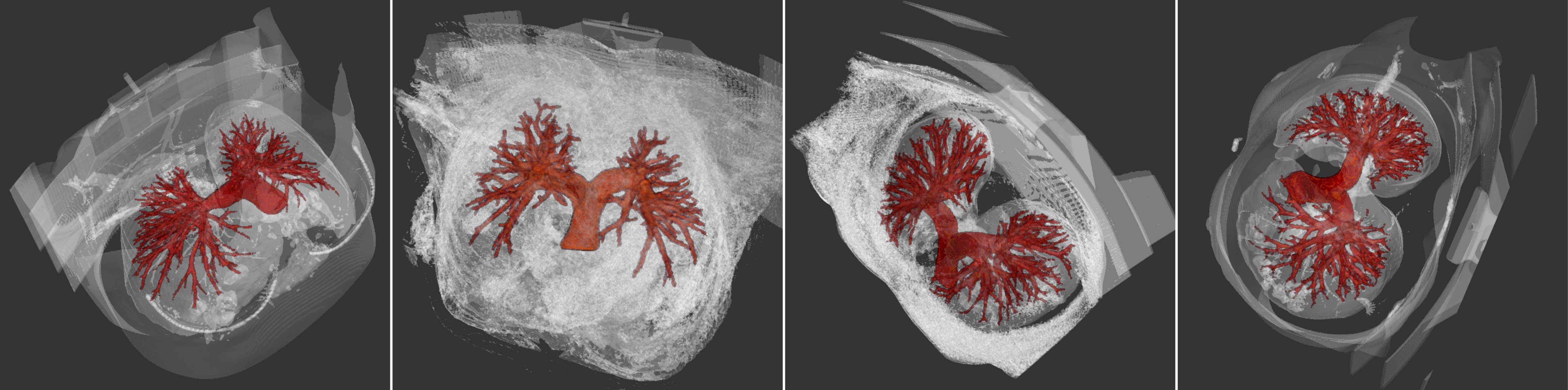}
\caption{PARSE dataset visualizations with ground truth segmentations in red; subjects 005, 082, 120, and 285 shown from left to right.}
\label{fig:parsevis}
\end{figure}

\begin{figure}[!htb]
\centering
\includegraphics[width=3.5in]{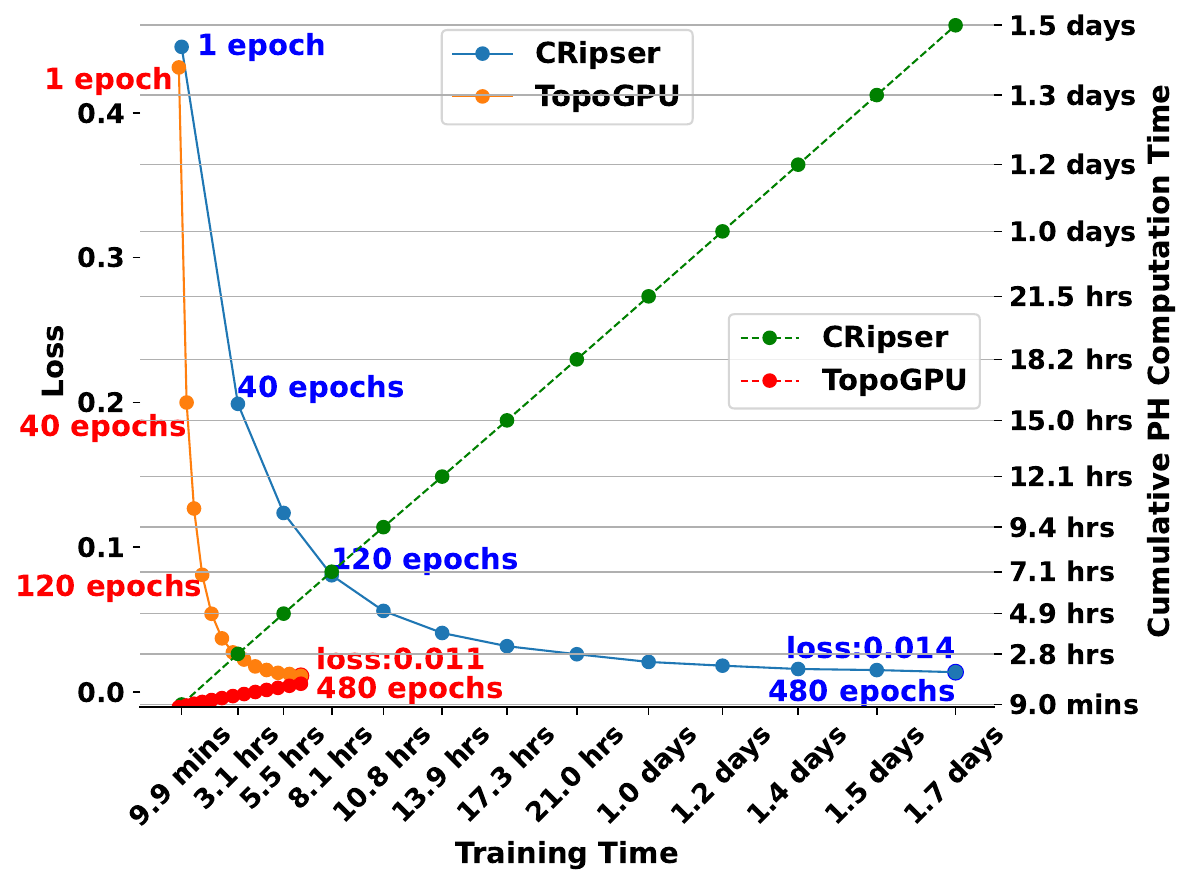}
\caption{Left y-axis: training loss versus training time. Right y-axis: cumulative persistent homology (PH) computation time versus training time. Both networks trained for 480 epochs.}
\label{fig:dltrain}
\end{figure}

We conduct a comparative analysis of Hu et al.'s topology-preserving frameworks equipped with TopoGPU versus with CRipser (the framework's default persistent homology tool) on an Ubuntu 18.04 multi-GPU workstation with NVIDIA GTX 1080 GPUs (2{,}560 CUDA cores and 12 GB GDDR5 memory). TopoGPU is compiled with CUDA 12.1 and \texttt{g++} 7.5.0, and wrapped via Cython \cite{cython} for Python interoperability. Cubicle is excluded from the comparison due to its lack of support for floating-point inputs. During training, one GPU handles network parameters and backpropagation while a second GPU is dedicated to persistent homology computation.

Experiments use the Pulmonary Artery Segmentation (PARSE) dataset \cite{parse}, designed for pulmonary artery structure studies in medical imaging. The dataset is visualized in Figure~\ref{fig:parsevis}. Each volume is resized to \(128 \times 128 \times 128\). The network is trained with a batch size of 4 over 480 epochs, computing persistent homology per epoch on 96 distinct \(32 \times 32 \times 32\) likelihood maps and the corresponding 96 \(32 \times 32 \times 32\) segmentation masks.

As illustrated in Figure~\ref{fig:dltrain}, the TopoGPU-accelerated framework converges to a loss of 0.011 in 6.45 hours (orange curve), averaging 48.56 seconds per epoch. In contrast, the framework employing CRipser converges to 0.014 after 1.71 days (blue curve), averaging 309.76 seconds per epoch. TopoGPU accounts for only \textbf{19.09\%} of the total training time, equating to \textbf{1.23 hours} (red curve), whereas CRipser dominates the process, consuming \textbf{87.31\%} of the training time and requiring \textbf{35.8 hours} (green curve). These findings highlight that TopoGPU substantially reduces the computational burden of persistent homology, so that it is no longer the dominant cost in this training workflow.

\section{Limitations and Future Work}
\label{sec:conclusion}
We introduce TopoGPU, a GPU streaming algorithm for persistent homology computation on image data. TopoGPU proceeds in four stages: (1) cubical complex construction; (2) stratification-aware discrete Morse matching; (3) Morse boundary computation via parallel topological sorting and parallel path parity computation; and (4) CPU-based boundary matrix reduction.

Across all datasets, TopoGPU outperforms every baseline. It supports 8-, 16-, and 32-bit signed and unsigned integers as well as floating-point inputs. Integrated into a topology-preserving deep learning framework, TopoGPU substantially reduces the burden of persistent homology computation during training, demonstrating both efficiency and practicality.

\begin{table}[!htb]
\caption{Examples of TopoGPU failure cases.}
\centering
\begin{threeparttable}
\begin{tabular}{@{}cccc@{}}
\hline
\multirow{2}{*}{Dataset} & \multirow{2}{*}{Size} & \multicolumn{1}{c}{Precomp.} & \multicolumn{1}{c}{End-to-End} \\
& & (s) & (s) \\ 
\hline
\noalign{\smallskip}
UniformRand & $2048\times2048\times2048$ & 231.2\text{\scriptsize$\pm$0.24} & \xmark \\
Binary    & $2048\times2048\times2048$ & 156.2\text{\scriptsize$\pm$0.20} & \xmark \\
Constant  & $2048\times2048\times2048$ & 147.0\text{\scriptsize$\pm$0.17} & \xmark \\ 
\noalign{\smallskip}
\hline
\end{tabular}
\label{tab:failcases}
\end{threeparttable}
\end{table}

TopoGPU can process inputs larger than GPU memory by accumulating boundary relations in host memory or, in principle, streaming them to disk. However, boundary matrix reduction must materialize the full matrix in memory; if it exceeds available RAM, reduction fails. Table~\ref{tab:failcases} reports failure cases of TopoGPU on three datasets---Uniform Random Noise (UniformRand), Binary data containing randomly distributed 0s and 1s, and Constant data filled entirely with 1s, each with 10 volumes of 32 GB.

TopoGPU precomputation covers the GPU-accelerated portion of the pipeline, from cubical complex construction through path parity computation, and outputs boundary relations ready for boundary matrix reduction. On these volumes, precomputation finishes quickly, but matrix reduction fails as a single volume requires more than 201.9 GB of RAM to hold the boundary matrix. Parallelizing this step remains challenging: the Gaussian-elimination-like process in boundary matrix reduction is largely sequential and CPU-intensive, presenting significant obstacles for GPU parallelization. Separately, on very small inputs (e.g., \(32 \times 32\) images in our study), the overheads of TopoGPU, such as CUDA context creation, initialization, and host-device transfers, can exceed the actual computation cost, and a lightweight CPU method such as CRipser may finish faster in this regime. A detailed study of this regime is provided in Section~VI and Table~S.VI of the supplementary material.

\section*{Acknowledgments}
This work was supported in part by NSF Grants CCF-2318633, CCF-2144901, NIH Grants R01NS143143, and R01CA297843.

\bibliography{Ref01}
\bibliographystyle{IEEEtran}

\newpage
\begin{IEEEbiography}[{\includegraphics[width=1in,height=1.25in,clip,keepaspectratio]{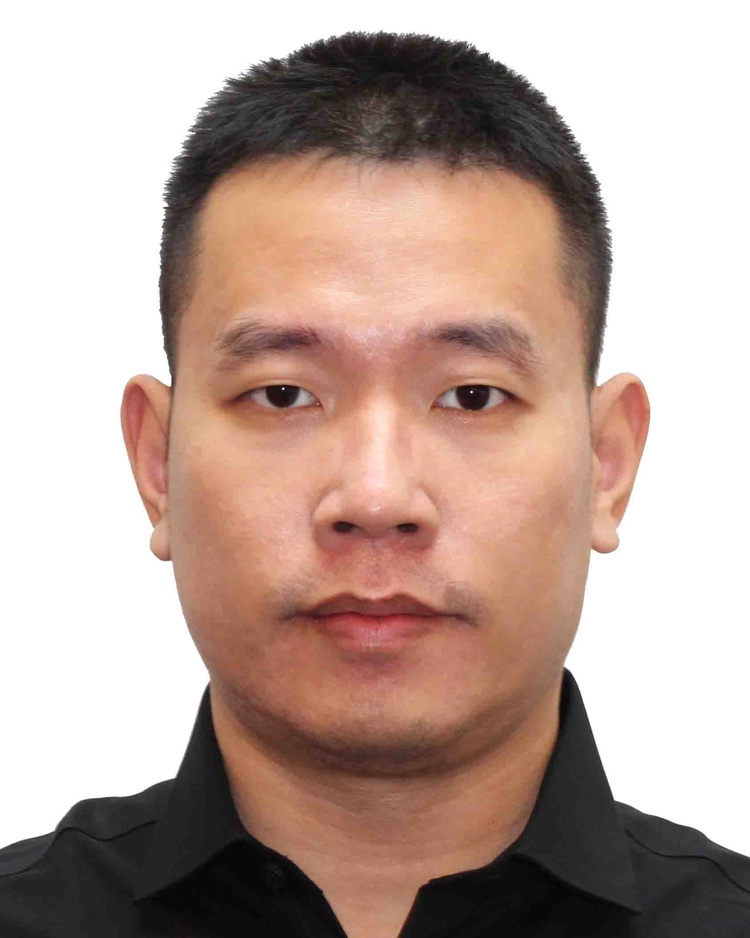}}]{Fan Wang}
received the BSc (Hons) in Computer Science from Lanzhou University in 2012 and the MSc in Electrical Engineering from the National University of Singapore in 2013. He joined the Advanced Digital Sciences Center in 2013 and then embarked on his PhD at Stony Brook University in Computer Science in 2016. His research and publications are centered on computational topology, GPU algorithms, generative deep learning networks, and 3D vision.
\end{IEEEbiography}

\begin{IEEEbiography}[{\includegraphics[width=1in,height=1.25in,clip,keepaspectratio]{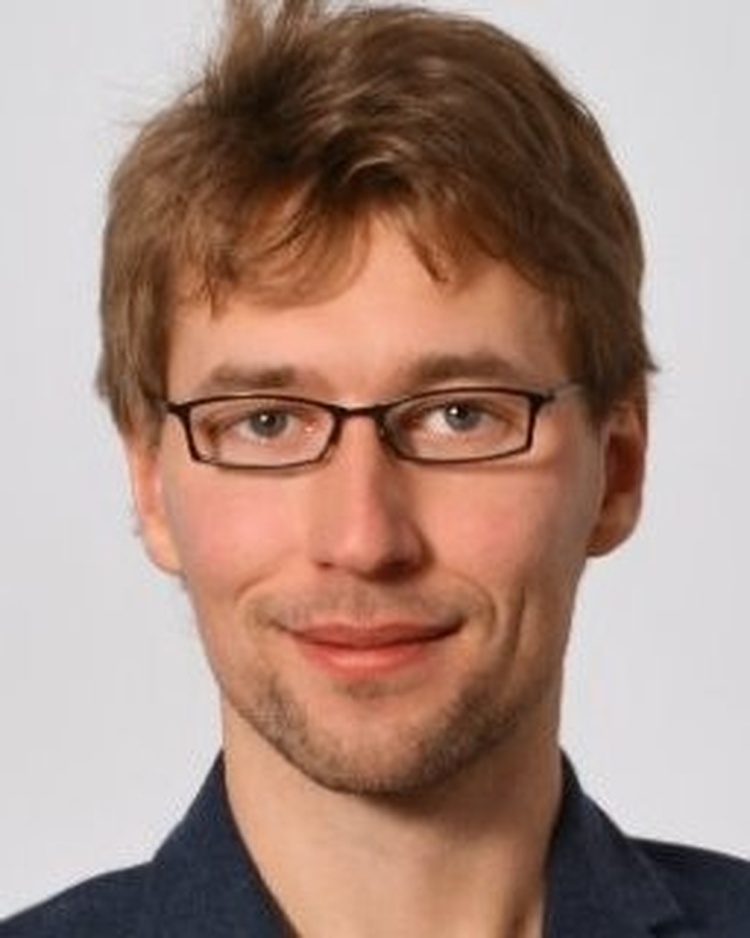}}]{Hubert Wagner}  
received his Msc (Hons) and PhD degrees from the Jagiellonian University in 2009 and 2014. 

Between 2015 and 2020 he did a postdoctoral fellowship at the Institute of Science and Technology Austria (ISTA). Since 2021 he has been an assistant professor at University of Florida. His research focuses on computational geometry and topology with applications.
\end{IEEEbiography}

\begin{IEEEbiography}[{\includegraphics[width=1in,height=1.25in,clip,keepaspectratio]{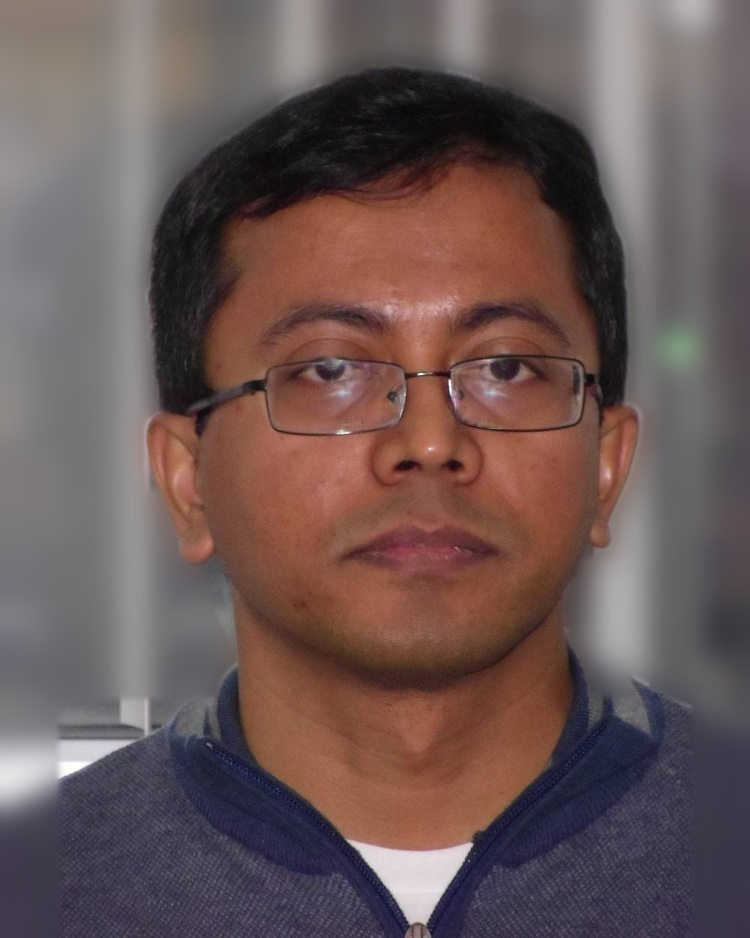}}]{Rezaul Chowdhury} is an Associate Professor of Computer Science at the State University of New York at Stony Brook. He holds a joint appointment with the Institute for Advanced Computational Science (IACS) and an affiliate appointment with the Department of Applied Mathematics and Statistics. He received his B.Sc. degree in Computer Science \& Engineering from Bangladesh University of Engineering \& Technology, and Ph.D. in Computer Science from the University of Texas at Austin. He is a recipient of an NSF CAREER award. His research interests include design and analysis of combinatorial algorithms and data structures, parallel and cache-efficient algorithms, algorithm engineering, quantitative finance, bioinformatics, and computational biology.
\end{IEEEbiography}

\begin{IEEEbiography}[{\includegraphics[width=1in,height=1.25in,clip,keepaspectratio]{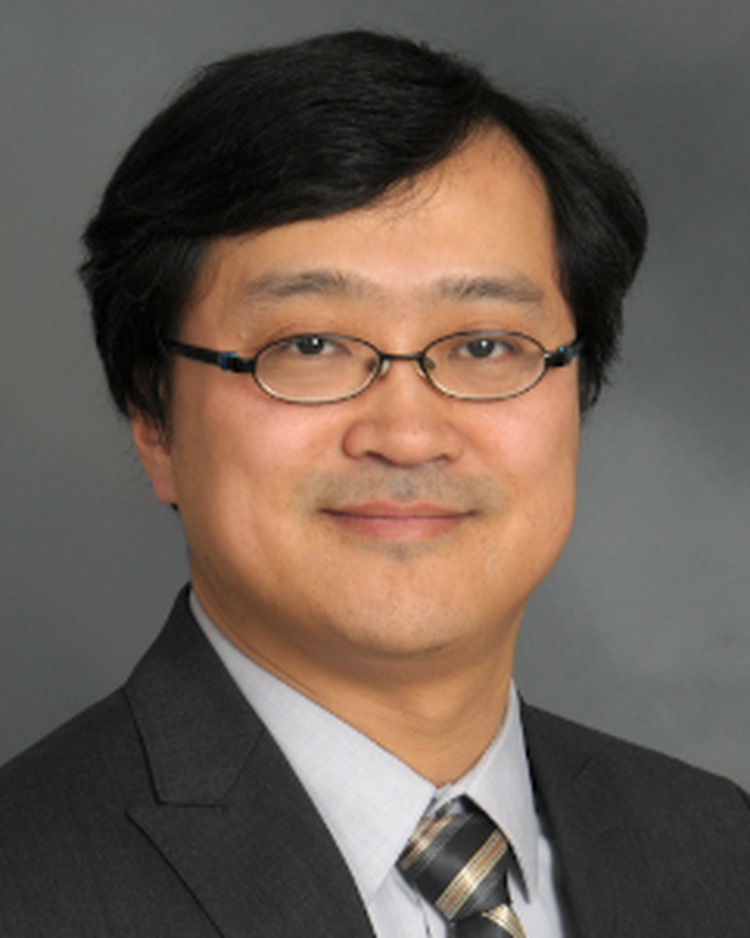}}]{Chao Chen} is an associate professor at Stony Brook University, Department of Biomedical Informatics. His research interest spans biomedical image analysis, topological data analysis (TDA) and machine learning. He develops principled learning methods to address problems in biomedical image analysis and machine learning from a unique topological view. His research is supported by multiple awards, including an NSF CAREER award to develop topology-driven learning methods.
\end{IEEEbiography}
\vfill

\clearpage


\section*{Supplementary material}

# Supplementary Material - GPU-Accelerated Computation of Persistent Homology for Topological Analysis of Image Data

Fan Wang, Hubert Wagner, Rezaul Chowdhury, Chao Chen

## I. Full End-to-End Runtime and Speedup Results

This section provides the full numerical results supporting the end-to-end runtime comparison reported in the main paper. Table S.I reports the complete runtime and throughput results for TopoGPU and all baseline methods on every evaluated dataset. Table S.II separately reports the corresponding speedups of TopoGPU over all baselines. We separate runtime and speedup results into two tables to improve readability while preserving the full evaluation data.

Across all datasets, TopoGPU achieves the lowest end-to-end runtime among all successful methods. Among the baselines, Cubicle is typically the strongest competitor, while CRipser often provides the next-best runtime on datasets where it completes. Nevertheless, TopoGPU consistently improves over both methods. Against CRipser, TopoGPU reaches up to $198.01\times$ speedup, with an average speedup of $53.63\times$. Relative to the fastest successful baseline for each dataset, TopoGPU reaches up to $71.95\times$ speedup, with an average speedup of $10.19\times$.

The full results also highlight the scalability limitations of the non-streaming baselines. DIPHA, GUDHI, and CRipser fail on several large or complex inputs because their non-streaming designs can require boundary matrices that exceed available memory. In contrast, TopoGPU remains successful across the full benchmark suite. Although the boundary matrix produced by TopoGPU is larger than that of the Morse complex without partitioning, TopoGPU still substantially reduces the matrix size relative to the original cubical complex. On average, the cubical-complex boundary matrix is $321.46\times$ larger than the TopoGPU boundary matrix, while the TopoGPU boundary matrix is $8.99\times$ larger than the Morse complex without partitioning.

The throughput column in Table S.I provides an additional view of computational efficiency. Throughput is measured in millions of voxels processed per second (MVox/s). TopoGPU achieves an average throughput of 9.97 MVox/s over all datasets, with the highest throughput of 32.90 MVox/s on the Aneurysm volume. These results show that TopoGPU provides both lower end-to-end runtime and robust throughput across small, medium, large, and complex volumetric datasets.

## II. Boundary Matrix Size Comparison with Non-partitioned Morse Complex

We compare the boundary matrix size produced by TopoGPU with two reference constructions: the original cubical complex and the Morse complex without partitioning. This comparison is used to quantify the overhead introduced by TopoGPU's stratification-aware Morse matching, where auxiliary critical cells are added to enable independent GPU-block processing and streaming.

To obtain the non-partitioned Morse complex baseline, we run Cubicle with the chunk number manually set to one, forcing the entire volume to be processed as a single chunk. This removes the auxiliary critical cells introduced by partitioning and therefore provides a direct comparison against a Morse complex without partitioning. However, this baseline requires the full dataset to be loaded and processed at once, and therefore does not scale to several large datasets evaluated by TopoGPU.

Table S.III reports the boundary matrix sizes of the cubical complex, the non-partitioned Morse complex, and the Morse complex produced by TopoGPU. As expected, TopoGPU produces larger boundary matrices than the non-partitioned Morse complex because stratification-aware matching introduces additional auxiliary critical cells. Across datasets where the non-partitioned baseline is available, the average TopoGPU/Morse-without-partitioning ratio is $8.99\times$.

Nevertheless, the TopoGPU boundary matrix remains substantially smaller than the cubical complex boundary matrix. On average, the cubical complex boundary matrix is $321.46\times$ larger than the TopoGPU boundary matrix. This shows that although partitioning introduces additional critical cells, TopoGPU still preserves the main computational benefit of Morse reduction: it greatly reduces the boundary matrix size compared with the original cubical complex while enabling scalable GPU streaming on large image data.

## III. Full Stage-wise Runtime Breakdown

This section provides the full stage-wise runtime breakdown supporting the summary reported in the main paper. Table S.IV reports the complete results over all evaluated datasets, including the runtime contribution of each TopoGPU stage and the corresponding peak RAM usage.

TABLE S.I

FULL END-TO-END RUNTIME AND THROUGHPUT RESULTS FOR PERSISTENT HOMOLOGY COMPUTATION.

| Dataset | Avg. #Persist. Pairs | **TopoGPU** (secs) | Cubicle (secs) | CRipser (secs) | DIPHA (secs) | GUDHI (secs) | TopoGPU MVox/s[a] |
|---|---|---|---|---|---|---|---|
| Silicium | 479 | 0.02±0.00 | 0.62±0.00 | 0.11±0.00 | 0.66±0.00 | 0.56±0.03 | 5.66 |
| Fuel | 126 | 0.03±0.00 | 0.65±0.04 | 0.28±0.00 | 1.43±0.03 | 1.33±0.00 | 8.74 |
| Neghip | 346 | 0.03±0.00 | 0.76±0.02 | 0.31±0.00 | 1.48±0.02 | 1.36±0.03 | 8.74 |
| Tooth | 204769 | 0.59±0.01 | 2.29±0.03 | 3.80±0.11 | 9.90±0.10 | 11.19±0.17 | 2.64 |
| Hydrogen | 7 | 0.12±0.00 | 0.89±0.03 | 3.67±0.07 | 12.04±0.31 | 12.16±0.04 | 17.48 |
| Shockwave | 26 | 0.11±0.00 | 0.78±0.01 | 4.25±0.09 | 11.88±0.46 | 11.09±0.04 | 19.06 |
| Lobster | 140275 | 0.86±0.00 | 4.56±0.02 | 13.02±0.18 | 35.42±0.08 | 37.37±0.44 | 6.35 |
| Ventricles | 1727530 | 4.23±0.04 | 14.50±0.24 | 30.81±0.55 | 54.56±0.77 | 76.06±0.11 | 1.92 |
| Engine | 242894 | 1.31±0.01 | 5.36±0.05 | 24.40±0.14 | 52.86±0.69 | 60.59±0.48 | 6.40 |
| Bonsai | 109008 | 0.70±0.01 | 3.49±0.04 | 76.51±5.25 | 99.45±0.62 | 115.3±1.08 | 23.97 |
| Aneurysm | 34494 | 0.51±0.01 | 1.87±0.02 | 54.68±1.44 | 96.05±0.36 | 111.7±1.53 | 32.90 |
| Foot | 749676 | 1.77±0.01 | 7.25±0.05 | 65.80±0.29 | 102.3±1.60 | 123.4±1.25 | 9.48 |
| Turbulence | 376226 | 1.15±0.00 | N/A[b] | 48.73±0.71 | 129.0±3.16 | 96.43±0.58 | 14.60 |
| Skull | 1502436 | 4.50±0.05 | 15.60±0.04 | 85.19±4.77 | 114.5±1.51 | 147.9±0.62 | 3.73 |
| Angio | 5761984 | 13.99±0.06 | 41.48±0.79 | 120.7±5.48 | 165.7±2.76 | 279.4±1.65 | 1.71 |
| Heptane | 27844 | 0.88±0.00 | 2.78±0.04 | 110.7±1.73 | 168.0±3.25 | 190.5±1.10 | 31.30 |
| Vertebra | 2782360 | 10.57±0.14 | 38.79±0.34 | 503.4±17.96 | 1135.8±6.50 | 1342.8±15.61 | 12.70 |
| Zeiss | 423233 | 12.87±0.15 | 48.02±0.58 | 2422.9±214.8 | 4166.6±195.1 | ✗[c] | 24.43 |
| Prone | 14038282 | 36.70±0.31 | 163.2±4.06 | 566.6±19.12 | 973.2±18.21 | 1300.1±2.28 | 3.31 |
| Neocortical | 12293475 | 47.24±0.80 | 188.1±6.00 | 7347.3±911.1 | 879.1±15.91 | ✗ | 2.43 |
| Present | 15446029 | 56.62±1.00 | 202.4±6.46 | 2108.0±80.97 | 1140.5±0.73 | ✗ | 1.89 |
| Stent | 3183900 | 7.53±0.04 | 25.16±0.19 | 140.2±3.03 | 387.9±2.81 | 378.2±4.12 | 6.02 |
| Christmas_tree | 20742396 | 70.85±2.16 | 287.0±1.79 | 979.1±11.19 | 1438.8±153.1 | 1478.2±7.98 | 1.85 |
| Marmoset | 43599571 | 251.1±2.00 | 834.8±10.50 | ✗ | ✗ | ✗ | 1.31 |
| Kingsnake | 38808043 | 242.6±2.00 | 803.1±4.83 | ✗ | ✗ | ✗ | 3.44 |
| Pawpawsaurus | 90157269 | 449.9±5.69 | 1395.7±20.95 | ✗ | ✗ | ✗ | 1.50 |
| Chameleon | 52031009 | 238.2±7.43 | 807.7±7.82 | ✗ | ✗ | ✗ | 4.75 |
| Richtmyer | 15725861 | 398.7±19.75 | 1110.2±23.23 | ✗ | ✗ | ✗ | 20.20 |
| Dense | 136899 | 4.79±0.22 | N/A | 208.2±6.26 | 203.0±5.16 | 138.5±2.28 | 4.34 |
| Fatty | 516830 | 7.60±0.13 | N/A | 546.79±11.84 | 762.5±7.16 | 906.2±30.79 | 14.74 |
| Hetero | 231627 | 6.76±0.25 | N/A | 379.7±13.80 | 287.3±2.50 | 363.4±26.86 | 5.70 |
| Scattered | 340780 | 10.47±0.32 | N/A | 825.8±92.05 | 486.0±5.25 | 779.8±16.90 | 7.11 |
| $\text{GRF}^{\alpha=1.0}$ | 42548195 | 160.0±2.48 | 525.0±10.95 | 1295.7±47.19 | 2233.7±198.0 | 2353.7±62.40 | 0.84 |
| $\text{GRF}^{\alpha=2.5}$ | 29496809 | 94.62±0.66 | 359.0±3.71 | 1029.2±17.13 | 1420.0±24.0 | 2079.2±30.20 | 1.42 |
| $\text{GRF}^{\alpha=5.0}$ | 1090669 | 7.22±0.03 | 27.70±2.98 | 919.5±18.46 | 1212.0±6.09 | 1294.4±106.0 | 18.59 |
| $\text{GRF}^{\alpha=7.5}$ | 27 | 4.87±0.00 | 11.51±4.03 | 964.3±13.64 | 1170.7±16.41 | 1135.9±23.93 | 27.56 |
| Average[d] | — | — | — | — | — | — | 9.97 |

[a] TopoGPU MVox/s: throughput in millions of voxels per second.
[b] N/A: Cubicle does not support floating-point inputs.
[c] ✗: DIPHA, GUDHI, and CRipser failed on large or complex inputs.
[d] Averages are over the full suite of 36 datasets and thus differ slightly from those in Table I of the main manuscript, which cover a representative subset.

Overall, boundary matrix reduction is the dominant stage, accounting for 54.31% of runtime on average, followed by path-parity computation at 25.63% and host-device data transfer at 7.02%. Topological sorting, Morse matching, and I/O contribute comparatively little to the total runtime. The full table also shows that small and medium volumes can be path-parity dominated, such as Shockwave at 78.76% and Bonsai at 48.51%, whereas large and complex volumes are typically reduction dominated, such as Pawpawsaurus at 61.59% and Chameleon at 61.19%. Device-host transfer time is mainly due to transferring boundary relations from the device to the host, and its contribution varies substantially across datasets, ranging from negligible on Silicium to 35.02% on Marmoset.

The full peak-memory measurements further confirm that RAM usage scales with the resulting boundary matrix size. Across the evaluated datasets, peak RAM ranges from 0.19 GB on the smallest volume to 54.67 GB on Richtmyer. As discussed in the main paper, TopoGPU has the second-lowest peak RAM usage among the evaluated methods, with only Cubicle using less RAM by offloading boundary relations to external storage.

## IV. RUNTIME SCALING WITH TOPOLOGICAL COMPLEXITY

Table I of the main manuscript reports TopoGPU's end-to-end runtime alongside the number of persistence pairs for each dataset. To characterize how this runtime scales with the size of the underlying computation, we plot it against two measures of topological complexity: the number of critical cells and the boundary matrix size (the number of ones over $\mathbb{Z}_2$ in the discrete Morse boundary matrix). Both quantities directly determine the dimensions and fill of the matrix that must be reduced, and are therefore more faithful predictors of computational cost than the raw input size. The results, for the first 16 datasets in Table I of the main manuscript, are shown in Figure S.1.

On log-log axes, runtime exhibits a clear near-linear relationship with both the number of critical cells and the boundary matrix size, spanning roughly three orders of magnitude in each. This behavior is consistent with the analysis in Section V-B of the main manuscript: because boundary matrix reduction

TABLE S.II
FULL SPEEDUP RESULTS OF TOPOGPU OVER EACH BASELINE.

| Dataset | Avg. #Persist. Pairs | Topo/Cubicle Speedup | Topo/CRipser Speedup | Topo/DIPHA Speedup | Topo/GUDHI Speedup | Topo/best Speedup[a] |
|---|---|---|---|---|---|---|
| Silicium | 479 | 31.00 | 5.50 | 33.00 | 28.00 | 5.50 |
| Fuel | 126 | 21.67 | 9.33 | 47.67 | 44.33 | 9.33 |
| Neghip | 346 | 25.33 | 10.33 | 49.33 | 45.33 | 10.33 |
| Tooth | 204769 | 3.88 | 6.44 | 16.78 | 18.97 | 3.88 |
| Hydrogen | 7 | 7.42 | 30.58 | 100.33 | 101.33 | 7.42 |
| Shockwave | 26 | 7.09 | 38.64 | 108.00 | 100.82 | 7.09 |
| Lobster | 140275 | 5.30 | 15.14 | 41.19 | 43.45 | 5.30 |
| Ventricles | 1727530 | 3.43 | 7.28 | 12.90 | 17.98 | 3.43 |
| Engine | 242894 | 4.09 | 18.63 | 40.35 | 46.25 | 4.09 |
| Bonsai | 109008 | 4.99 | 109.30 | 142.07 | 164.71 | 4.99 |
| Aneurysm | 34494 | 3.67 | 107.22 | 188.33 | 219.02 | 3.67 |
| Foot | 749676 | 4.10 | 37.18 | 57.80 | 69.72 | 4.10 |
| Turbulence | 376226 | - | 42.37 | 112.17 | 83.85 | 42.37 |
| Skull | 1502436 | 3.47 | 18.93 | 25.44 | 32.87 | 3.47 |
| Angio | 5761984 | 2.96 | 8.63 | 11.84 | 19.97 | 2.96 |
| Heptane | 27844 | 3.16 | 125.80 | 190.91 | 216.48 | 3.16 |
| Vertebra | 2782360 | 3.67 | 47.63 | 107.46 | 127.04 | 3.67 |
| Zeiss | 423233 | 3.73 | 188.26 | 323.75 | - | 3.73 |
| Prone | 14038282 | 4.45 | 15.44 | 26.52 | 35.43 | 4.45 |
| Neocortical | 12293475 | 3.98 | 155.53 | 18.61 | - | 3.98 |
| Present | 15446029 | 3.57 | 37.23 | 20.14 | - | 3.57 |
| Stent | 3183900 | 3.34 | 18.62 | 51.51 | 50.23 | 3.34 |
| Christmas_tree | 20742396 | 4.05 | 13.82 | 20.31 | 20.86 | 4.05 |
| Marmoset | 43599571 | 3.32 | - | - | - | 3.32 |
| Kingsnake | 38808043 | 3.31 | - | - | - | 3.31 |
| Pawpawsaurus | 90157269 | 3.10 | - | - | - | 3.10 |
| Chameleon | 52031009 | 3.39 | - | - | - | 3.39 |
| Richtmyer | 15725861 | 2.78 | - | - | - | 2.78 |
| Dense | 136899 | - | 43.47 | 42.38 | 28.91 | 28.91 |
| Fatty | 516830 | - | 71.95 | 100.33 | 119.24 | 71.95 |
| Hetero | 231627 | - | 56.17 | 42.50 | 53.76 | 42.50 |
| Scattered | 340780 | - | 78.87 | 46.42 | 74.48 | 46.42 |
| GRF$^{\alpha=1.0}$ | 42548195 | 3.28 | 8.10 | 13.96 | 14.71 | 3.28 |
| GRF$^{\alpha=2.5}$ | 29496809 | 3.79 | 10.88 | 15.01 | 21.97 | 3.79 |
| GRF$^{\alpha=5.0}$ | 1090669 | 3.84 | 127.35 | 167.87 | 179.28 | 3.84 |
| GRF$^{\alpha=7.5}$ | 27 | 2.36 | 198.01 | 240.39 | 233.24 | 2.36 |
| Average[b] | — | 6.05 | 53.63 | 77.91 | 79.01 | 10.19 |

[a] Topo/best Speedup: TopoGPU's speedup over the fastest baseline.
[b] Averages are over the full suite of 36 datasets and thus differ slightly from those in Table I of the main manuscript, which cover a representative subset.

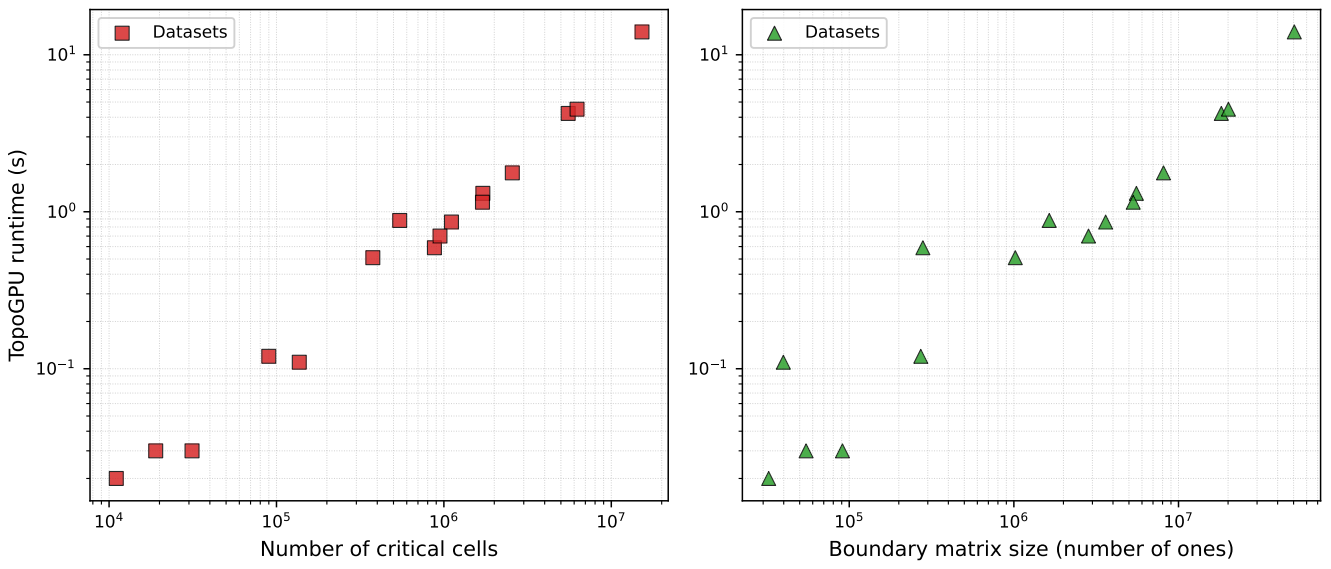


Fig. S.1. TopoGPU end-to-end runtime as a function of the number of critical cells (left) and the boundary matrix size (right), for the first 16 datasets in Table I of the main manuscript. Each marker denotes one dataset. Both axes are on a logarithmic scale.

is the dominant stage of the pipeline, and its cost is governed by the dimensions and fill of the Morse boundary matrix, TopoGPU's total runtime tracks these two complexity measures closely. Raw input size, by contrast, is a weaker predictor, since two volumes of identical voxel dimensions can differ substantially in topological complexity and hence in the size of the boundary matrix they produce. These plots therefore reinforce our conclusion that the size of the resulting Morse complex, rather than the number of input voxels, is the primary determinant of TopoGPU's computational cost.

## V. EMPIRICAL ITERATION COUNTS OF PARALLEL TOPOLOGICAL SORTING

To further evaluate the practical behavior of the proposed parallel topological sorting algorithm, we report the number of iterations required on benchmark matching graphs. For each dataset, we record the number of vertices $N$ in the matching graph and the number of iterations $r$ required by parallel topological sorting to complete the topological ordering. The results are summarized in Table S.V.

As shown in Table S.V, the matching graphs vary substantially in size, ranging from $9.4 \cdot 10^5$ vertices for Silicium to $2.2 \cdot 10^8$ vertices for Heptane. Despite this large variation in graph size, the number of iterations remains small and stable across all datasets. Specifically, the algorithm requires only 11 to 17 iterations, with an average of 13.25 iterations over the 16 benchmark graphs.

TABLE S.III
**Boundary matrix size comparison.** Boundary matrix sizes of the cubical complex, the Morse complex without partitioning, and the Morse complex by TopoGPU.

| **Dataset** | **Size** | **Boundary matrix size** | | | **Cubical / TopoGPU** | **TopoGPU / Morse w/o partitioning** |
|---|---|---|---|---|---|---|
| | | **Cubical complex** | **Morse complex w/o partitioning** | **Morse complex TopoGPU** | | |
| Silicium | $98 \times 34 \times 34$ | $2.78 \cdot 10^6$ | $2.06 \cdot 10^3$ | $3.24 \cdot 10^4$ | 85.80 | 15.73 |
| Fuel | $64 \times 64 \times 64$ | $6.39 \cdot 10^6$ | $6.23 \cdot 10^2$ | $5.47 \cdot 10^4$ | 116.82 | 87.80 |
| Neghip | $64 \times 64 \times 64$ | $6.39 \cdot 10^6$ | $6.48 \cdot 10^3$ | $9.09 \cdot 10^4$ | 70.30 | 14.03 |
| Tooth | $103 \times 94 \times 161$ | $3.77 \cdot 10^7$ | $1.32 \cdot 10^6$ | $2.80 \cdot 10^6$ | 13.46 | 2.12 |
| Hydrogen | $128 \times 128 \times 128$ | $5.07 \cdot 10^7$ | $3.32 \cdot 10^4$ | $2.72 \cdot 10^5$ | 186.40 | 8.19 |
| Shockwave | $64 \times 64 \times 512$ | $5.09 \cdot 10^7$ | $2.10 \cdot 10^3$ | $3.98 \cdot 10^4$ | 1278.89 | 18.95 |
| Lobster | $301 \times 324 \times 56$ | $1.32 \cdot 10^8$ | $1.59 \cdot 10^6$ | $3.61 \cdot 10^6$ | 36.57 | 2.27 |
| Ventricles | $256 \times 256 \times 124$ | $1.96 \cdot 10^8$ | $9.85 \cdot 10^6$ | $1.82 \cdot 10^7$ | 10.77 | 1.85 |
| Engine | $256 \times 256 \times 128$ | $2.02 \cdot 10^8$ | $2.44 \cdot 10^6$ | $5.55 \cdot 10^6$ | 36.40 | 2.27 |
| Bonsai | $256 \times 256 \times 256$ | $4.04 \cdot 10^8$ | $7.31 \cdot 10^5$ | $2.84 \cdot 10^6$ | 142.25 | 3.89 |
| Aneurysm | $256 \times 256 \times 256$ | $4.04 \cdot 10^8$ | $1.00 \cdot 10^5$ | $1.02 \cdot 10^6$ | 396.08 | 10.20 |
| Foot | $256 \times 256 \times 256$ | $4.04 \cdot 10^8$ | $3.64 \cdot 10^6$ | $8.11 \cdot 10^6$ | 49.82 | 2.23 |
| Turbulence | $256 \times 256 \times 256$ | $4.04 \cdot 10^8$ | – | $5.31 \cdot 10^6$ | 76.08 | – |
| Skull | $256 \times 256 \times 256$ | $4.04 \cdot 10^8$ | $9.13 \cdot 10^6$ | $2.01 \cdot 10^7$ | 20.10 | 2.20 |
| Angio | $416 \times 512 \times 112$ | $5.75 \cdot 10^8$ | $2.93 \cdot 10^7$ | $5.05 \cdot 10^7$ | 11.39 | 1.72 |
| Heptane | $302 \times 302 \times 302$ | $6.63 \cdot 10^8$ | $2.45 \cdot 10^5$ | $1.64 \cdot 10^6$ | 404.27 | 6.69 |
| Vertebra | $512 \times 512 \times 512$ | $3.23 \cdot 10^9$ | $1.43 \cdot 10^7$ | $4.14 \cdot 10^7$ | 78.03 | 2.90 |
| Zeiss | $680 \times 680 \times 680$ | $7.56 \cdot 10^9$ | $1.56 \cdot 10^7$ | $4.62 \cdot 10^7$ | 163.74 | 2.97 |
| Prone | $512 \times 512 \times 463$ | $2.92 \cdot 10^9$ | $6.69 \cdot 10^7$ | $1.29 \cdot 10^8$ | 22.55 | 1.94 |
| Neocortical | $1464 \times 1033 \times 76$ | $2.77 \cdot 10^9$ | $1.08 \cdot 10^8$ | $1.19 \cdot 10^8$ | 23.27 | 1.10 |
| Present | $492 \times 492 \times 442$ | $2.40 \cdot 10^9$ | $1.17 \cdot 10^8$ | $1.41 \cdot 10^8$ | 17.06 | 1.20 |
| Stent | $512 \times 512 \times 174$ | $1.10 \cdot 10^9$ | $1.60 \cdot 10^7$ | $3.13 \cdot 10^7$ | 35.12 | 1.95 |
| Christmas_tree | $512 \times 499 \times 512$ | $3.15 \cdot 10^9$ | $1.19 \cdot 10^8$ | $2.06 \cdot 10^8$ | 15.28 | 1.74 |
| Marmoset | $1024 \times 1024 \times 314$ | $7.92 \cdot 10^9$ | $3.14 \cdot 10^8$ | $5.40 \cdot 10^8$ | 14.65 | 1.72 |
| Kingsnake | $1024 \times 1024 \times 795$ | $2.00 \cdot 10^{10}$ | $2.69 \cdot 10^8$ | $5.79 \cdot 10^8$ | 34.54 | 2.15 |
| Pawpawsaurus | $958 \times 646 \times 1088$ | $1.62 \cdot 10^{10}$ | $4.14 \cdot 10^8$ | $8.31 \cdot 10^8$ | 19.48 | 2.01 |
| Chameleon | $1024 \times 1024 \times 1080$ | $2.72 \cdot 10^{10}$ | $2.36 \cdot 10^8$ | $5.82 \cdot 10^8$ | 46.76 | 2.47 |
| Richtmyer | $2048 \times 2048 \times 1920$ | $1.93 \cdot 10^{11}$ | – | $4.28 \cdot 10^8$ | 450.41 | – |
| Dense | $287 \times 359 \times 202$ | $5.01 \cdot 10^8$ | – | $8.53 \cdot 10^6$ | 58.72 | – |
| Fatty | $440 \times 518 \times 488$ | $2.67 \cdot 10^9$ | – | $2.18 \cdot 10^7$ | 122.63 | – |
| Hetero | $359 \times 359 \times 299$ | $9.28 \cdot 10^8$ | – | $1.32 \cdot 10^7$ | 70.22 | – |
| Scattered | $434 \times 446 \times 384$ | $1.66 \cdot 10^9$ | – | $2.14 \cdot 10^7$ | 77.61 | – |
| $\mathrm{GRF}^{\alpha=1.0}$ | $512 \times 512 \times 512$ | $3.23 \cdot 10^9$ | $1.98 \cdot 10^8$ | $3.38 \cdot 10^8$ | 9.56 | 1.70 |
| $\mathrm{GRF}^{\alpha=2.5}$ | $512 \times 512 \times 512$ | $3.23 \cdot 10^9$ | $1.31 \cdot 10^8$ | $2.48 \cdot 10^8$ | 13.01 | 1.89 |
| $\mathrm{GRF}^{\alpha=5.0}$ | $512 \times 512 \times 512$ | $3.23 \cdot 10^9$ | $5.41 \cdot 10^6$ | $2.64 \cdot 10^7$ | 122.48 | 4.87 |
| $\mathrm{GRF}^{\alpha=7.5}$ | $512 \times 512 \times 512$ | $3.23 \cdot 10^9$ | $7.56 \cdot 10^3$ | $4.46 \cdot 10^5$ | 7242.15 | 58.99 |
| Average | – | – | – | – | 321.46 | 8.99 |

*Note.* The matrix sizes reported for the Morse complex with TopoGPU use the same block size as the TopoGPU timings reported in Tables I and II of the main paper. "Cubical / TopoGPU" reports the ratio between the cubical complex boundary matrix size and the TopoGPU boundary matrix size. "TopoGPU / Morse w/o partitioning" reports the ratio between the TopoGPU boundary matrix size and the Morse complex boundary matrix size without partitioning.

This empirical behavior indicates that image-derived Morse matching graphs do not exhibit large iteration counts in practice. Even for graphs with hundreds of millions of vertices, the algorithm completes topological sorting in a small number of iterations. These results support the practical efficiency of the proposed algorithm and show that the pessimistic large-iteration scenario does not occur on the evaluated matching graphs.

## VI. Performance on Small Inputs

The Limitations section of the main manuscript notes that, on very small inputs, the overheads of TopoGPU can exceed the actual computation cost, and a lightweight CPU method such as CRipser may finish faster. This section characterizes this regime, which arises from two factors.

- First, TopoGPU carries overheads including CUDA context creation, device memory allocation and initialization, and host–device transfers. On large volumes these costs are amortized over the computation, but on small inputs they dominate: a CPU method begins computing immediately, whereas TopoGPU spends most of its runtime preparing the device.
- Second, small inputs do not saturate the GPU. TopoGPU assigns one thread per voxel, so a small input occupies only a fraction of the available streaming multiprocessors, leaving a large portion of the device idle; the parallelism that drives TopoGPU's advantage on large volumes simply has no work to distribute.

The second factor also suggests a mitigation: batch processing. When multiple small inputs are processed concurrently, their combined workload saturates the GPU and the fixed overheads are amortized across the batch. This mirrors the deep learning setting of Section VII, where persistent homology is computed for a batch of likelihood maps in each training iteration rather than for a single input at a time. In our implementation, batching launches multiple independent TopoGPU

TABLE S.IV
STAGE-WISE RUNTIME BREAKDOWN AND PEAK HOST RAM AND DEVICE MEMORY USAGE FOR TOPOGPU.

| Dataset | Size | Data Type | BoundMat Size[a] | Dev↔Host | Morse Matching | Topo Sort | Path Parity | Matrix Reduction | I/O | RAM[b] (GB) | VRAM[c] (GB) |
|---|---|---|---|---|---|---|---|---|---|---|---|
| Silicium | $98 \times 34 \times 34$ | uint8 | $3.24 \cdot 10^4$ | 0.00% | 1.58% | 2.40% | 22.20% | 44.72% | 10.17% | 0.19 | 0.34 |
| Fuel | $64 \times 64 \times 64$ | uint8 | $5.47 \cdot 10^4$ | 0.00% | 1.48% | 3.12% | 47.76% | 28.02% | 14.62% | 0.19 | 0.31 |
| Neghip | $64 \times 64 \times 64$ | uint8 | $9.09 \cdot 10^4$ | 0.00% | 1.39% | 3.57% | 48.08% | 26.78% | 10.18% | 0.20 | 0.31 |
| Tooth | $103 \times 94 \times 161$ | uint8 | $2.80 \cdot 10^6$ | 0.19% | 0.63% | 1.25% | 15.38% | 77.85% | 3.54% | 0.42 | 0.53 |
| Hydrogen | $128 \times 128 \times 128$ | uint8 | $2.72 \cdot 10^5$ | 0.02% | 2.64% | 6.94% | 55.33% | 19.82% | 14.05% | 0.28 | 0.91 |
| Shockwave | $64 \times 64 \times 512$ | uint8 | $3.98 \cdot 10^4$ | 0.01% | 1.93% | 5.69% | 78.76% | 5.69% | 7.62% | 0.22 | 0.31 |
| Lobster | $301 \times 324 \times 56$ | uint8 | $3.61 \cdot 10^6$ | 0.28% | 0.82% | 2.22% | 15.36% | 74.37% | 2.95% | 0.62 | 1.74 |
| Ventricles | $256 \times 256 \times 124$ | uint8 | $1.82 \cdot 10^7$ | 1.31% | 0.54% | 1.35% | 10.11% | 84.31% | 1.01% | 1.62 | 0.91 |
| Engine | $256 \times 256 \times 128$ | uint8 | $5.55 \cdot 10^6$ | 0.41% | 0.71% | 2.50% | 16.83% | 71.51% | 3.96% | 0.65 | 0.91 |
| Bonsai | $256 \times 256 \times 256$ | uint8 | $2.84 \cdot 10^6$ | 0.22% | 4.72% | 7.49% | 48.51% | 26.44% | 10.12% | 0.50 | 0.91 |
| Aneurysm | $256 \times 256 \times 256$ | uint8 | $1.02 \cdot 10^6$ | 0.09% | 6.36% | 9.18% | 60.67% | 12.12% | 9.22% | 0.39 | 0.91 |
| Foot | $256 \times 256 \times 256$ | uint8 | $8.11 \cdot 10^6$ | 0.61% | 0.96% | 3.04% | 21.24% | 67.21% | 5.15% | 0.96 | 0.91 |
| Turbulence | $256 \times 256 \times 256$ | float32 | $5.31 \cdot 10^6$ | 0.40% | 1.06% | 5.40% | 31.99% | 45.39% | 12.25% | 0.78 | 0.92 |
| Skull | $256 \times 256 \times 256$ | uint8 | $2.01 \cdot 10^7$ | 1.38% | 0.46% | 1.67% | 18.50% | 74.17% | 2.23% | 1.74 | 0.91 |
| Angio | $416 \times 512 \times 112$ | uint16 | $5.05 \cdot 10^7$ | 3.61% | 0.32% | 0.98% | 10.61% | 83.61% | 0.60% | 4.10 | 1.91 |
| Heptane | $302 \times 302 \times 302$ | uint8 | $1.64 \cdot 10^6$ | 0.19% | 6.92% | 9.37% | 57.88% | 14.52% | 10.61% | 0.71 | 1.74 |
| Vertebra | $512 \times 512 \times 512$ | uint16 | $4.14 \cdot 10^7$ | 3.04% | 1.58% | 4.32% | 32.92% | 49.29% | 4.79% | 4.50 | 2.09 |
| Zeiss | $680 \times 680 \times 680$ | uint8 | $4.62 \cdot 10^7$ | 3.51% | 4.61% | 8.39% | 39.10% | 37.91% | 5.25% | 6.29 | 2.05 |
| Prone | $512 \times 512 \times 463$ | uint16 | $1.29 \cdot 10^8$ | 9.32% | 0.30% | 1.22% | 10.74% | 75.62% | 1.22% | 10.16 | 2.09 |
| Neocortical | $1464 \times 1033 \times 76$ | uint8 | $1.19 \cdot 10^8$ | 8.80% | 0.20% | 0.84% | 12.03% | 76.01% | 0.32% | 10.20 | 1.47 |
| Present | $492 \times 492 \times 442$ | uint16 | $1.41 \cdot 10^8$ | 10.65% | 0.31% | 1.18% | 6.22% | 74.17% | 0.63% | 10.98 | 2.02 |
| Stent | $512 \times 512 \times 174$ | uint16 | $3.13 \cdot 10^7$ | 2.44% | 0.89% | 2.14% | 12.31% | 72.85% | 2.37% | 3.10 | 2.09 |
| Christmas_tree | $512 \times 499 \times 512$ | uint16 | $2.06 \cdot 10^8$ | 15.65% | 0.41% | 1.50% | 8.33% | 73.24% | 0.74% | 14.90 | 2.09 |
| Marmoset | $1024 \times 1024 \times 314$ | uint8 | $5.40 \cdot 10^8$ | 35.02% | 0.28% | 1.12% | 6.56% | 55.78% | 0.24% | 38.94 | 2.06 |
| Kingsnake | $1024 \times 1024 \times 795$ | uint8 | $5.79 \cdot 10^8$ | 30.44% | 0.27% | 1.23% | 8.12% | 57.40% | 0.62% | 42.06 | 2.06 |
| Pawpawsaurus | $958 \times 646 \times 1088$ | uint16 | $8.31 \cdot 10^8$ | 27.36% | 0.29% | 1.12% | 6.39% | 61.59% | 0.30% | 54.04 | 1.96 |
| Chameleon | $1024 \times 1024 \times 1080$ | uint16 | $5.82 \cdot 10^8$ | 26.75% | 0.29% | 1.42% | 9.38% | 61.19% | 0.83% | 42.28 | 4.17 |
| Richtmyer | $2048 \times 2048 \times 1920$ | uint8 | $4.28 \cdot 10^8$ | 20.31% | 3.10% | 6.26% | 37.02% | 30.06% | 2.46% | 54.67 | 4.42 |
| Dense | $287 \times 359 \times 202$ | float32 | $8.53 \cdot 10^6$ | 0.80% | 1.03% | 2.46% | 16.41% | 74.97% | 3.05% | 1.15 | 1.56 |
| Fatty | $440 \times 518 \times 488$ | float32 | $2.18 \cdot 10^7$ | 1.82% | 2.74% | 6.58% | 33.47% | 47.78% | 7.11% | 2.57 | 2.00 |
| Hetero | $359 \times 359 \times 299$ | float32 | $1.32 \cdot 10^7$ | 1.07% | 0.99% | 2.76% | 15.43% | 73.84% | 3.11% | 1.62 | 1.56 |
| Scattered | $434 \times 446 \times 384$ | float32 | $2.14 \cdot 10^7$ | 1.77% | 1.00% | 3.70% | 18.33% | 70.74% | 3.85% | 2.55 | 1.73 |
| $\mathrm{GRF}^{\alpha=1.0}$ | $512 \times 512 \times 512$ | int32 | $3.38 \cdot 10^8$ | 24.19% | 0.11% | 0.39% | 8.50% | 66.49% | 0.31% | 24.54 | 2.16 |
| $\mathrm{GRF}^{\alpha=2.5}$ | $512 \times 512 \times 512$ | int32 | $2.48 \cdot 10^8$ | 17.88% | 0.33% | 1.19% | 9.70% | 69.92% | 0.53% | 18.40 | 2.16 |
| $\mathrm{GRF}^{\alpha=5.0}$ | $512 \times 512 \times 512$ | int32 | $2.64 \cdot 10^7$ | 2.59% | 1.50% | 8.06% | 33.27% | 46.04% | 7.26% | 3.18 | 2.16 |
| $\mathrm{GRF}^{\alpha=7.5}$ | $512 \times 512 \times 512$ | int32 | $4.46 \cdot 10^5$ | 0.59% | 7.62% | 12.38% | 39.31% | 23.91% | 14.19% | 1.15 | 2.16 |
| Average | — | — | — | 7.02% | 1.68% | 3.73% | 25.63% | 54.31% | 4.93% | — | — |

[a] BoundMat Size: the number of ones over $\mathbb{Z}_2$ in the discrete Morse boundary matrix.
[b] RAM: peak host (CPU) RAM used by TopoGPU.
[c] VRAM: peak device (GPU) memory used by TopoGPU.

instances as concurrent processes, each processing a disjoint subset of the inputs.

To quantify these effects, we compare CRipser, TopoGPU processing one input at a time, and TopoGPU processing inputs in batches, on synthetic 2D inputs of varying sizes and two complexity levels: uniform random inputs (high topological complexity) and constant inputs filled with ones (low topological complexity). For each input type and resolution, we generate 100 files, and the results, averaged over these 100 files, are reported in Table S.VI.

The results confirm both factors. At sizes of $256 \times 256$ and above, TopoGPU outperforms CRipser even when processing one input at a time, and batching widens the advantage. The speedup is largest on the constant inputs (e.g., 22× over CRipser at $1024 \times 1024$). On these inputs, nearly all of the work is completed by the GPU-accelerated stages, and the near-trivial Morse complex leaves almost nothing for the sequential boundary matrix reduction on the CPU.

Batching, in contrast, helps most on the random inputs, improving single-input processing by 2–4× (e.g., from 408.6 ms to 100.9 ms at $1024 \times 1024$), and only modestly on the constant inputs. This difference reflects where the runtime is spent. On random inputs, the large number of persistence pairs makes both the GPU stages and the subsequent boundary matrix reduction substantial, and processing multiple inputs concurrently both saturates the GPU and overlaps the CPU-side reduction of one input with the GPU stages of others. On constant inputs, the Morse complex is nearly trivial and the runtime is dominated by GPU overheads, which batching can only partially hide.

At the smallest size of $32 \times 32$, however, CRipser remains faster than TopoGPU in both modes: the fixed overheads outweigh any parallelism gains even under batching, which narrows the gap but does not close it. This is consistent with the observation in the Limitations section of the main manuscript: for very small inputs, a lightweight CPU method is the more efficient choice, while TopoGPU's advantage grows rapidly

TABLE S.V

**EMPIRICAL ITERATION COUNTS OF PARALLEL TOPOLOGICAL SORTING.** NUMBER OF VERTICES AND ITERATIONS REQUIRED BY THE PARALLEL TOPOLOGICAL SORTING ALGORITHM ON BENCHMARK MATCHING GRAPHS.

| Dataset | #Vertices $N$ | #Iterations $r$ |
|---|---|---|
| Silicium | $9.4 \cdot 10^5$ | 11 |
| Fuel | $2.1 \cdot 10^6$ | 12 |
| Neghip | $2.1 \cdot 10^6$ | 14 |
| Tooth | $1.3 \cdot 10^7$ | 12 |
| Hydrogen | $1.7 \cdot 10^7$ | 16 |
| Shockwave | $1.7 \cdot 10^7$ | 15 |
| Lobster | $4.4 \cdot 10^7$ | 12 |
| Ventricles | $6.6 \cdot 10^7$ | 11 |
| Engine | $6.8 \cdot 10^7$ | 14 |
| Bonsai | $1.4 \cdot 10^8$ | 13 |
| Aneurysm | $1.4 \cdot 10^8$ | 11 |
| Foot | $1.4 \cdot 10^8$ | 14 |
| Turbulence | $1.4 \cdot 10^8$ | 16 |
| Skull | $1.4 \cdot 10^8$ | 13 |
| Angio | $1.9 \cdot 10^8$ | 11 |
| Heptane | $2.2 \cdot 10^8$ | 17 |

*Note.* $N$ denotes the number of vertices in the matching graph, and $r$ denotes the number of iterations required by parallel topological sorting.

TABLE S.VI

RUNTIME COMPARISON ON 2D INPUTS OF VARYING SIZES AND TOPOLOGICAL COMPLEXITY: CRIPSER VERSUS TOPOGPU WITH SINGLE-INPUT AND BATCHED PROCESSING.

| Input | Size | CRipser (ms) | TopoGPU-single (ms) | TopoGPU-batch (ms) |
|---|---|---|---|---|
| UniformRand | $1024 \times 1024$ | 1236.6±8.03 | 408.6±14.1 | 100.9±6.32 |
| | $512 \times 512$ | 242.9±1.11 | 74.7±2.29 | 30.1±1.32 |
| | $256 \times 256$ | 39.5±0.81 | 20.9±1.24 | 10.7±0.65 |
| | $32 \times 32$ | 0.64±0.05 | 2.96±0.52 | 2.56±0.39 |
| Constant | $1024 \times 1024$ | 475.6±2.02 | 21.2±0.80 | 19.4±1.24 |
| | $512 \times 512$ | 84.1±1.01 | 9.94±0.64 | 8.37±0.63 |
| | $256 \times 256$ | 15.1±0.35 | 8.00±0.43 | 6.68±0.50 |
| | $32 \times 32$ | 0.38±0.01 | 2.79±0.47 | 2.23±0.39 |

with input size.

## VII. DEEP LEARNING PIPELINE INTEGRATION

To demonstrate the practicality of TopoGPU, we integrate it into the topology-preserving deep-learning segmentation workflow proposed by Hu et al. [1]. We compare the training time of the TopoGPU-based framework with that of an otherwise identical framework using Cubical Ripser (CRipser) [2] across three network architectures: 3D U-Net, Attention 3D U-Net, and Residual 3D U-Net, as detailed in Table S.VII.

The experiments are conducted on the following datasets:

- The Pulmonary Artery Segmentation (PARSE) dataset [3], which is designed for studying pulmonary artery structures in medical imaging. Each volume is resized to $128 \times 128 \times 128$.
- The MSD Brain dataset [4], [5], corresponding to Task 01 of the Medical Segmentation Decathlon, which focuses on segmenting three tumor sub-regions from multi-parametric magnetic resonance images. Each volume has a size of $240 \times 240 \times 155$. We merge the labels of Enhancing Tumor, Peritumoral Edema, and Necrotic Core into a single Whole Tumor label, and binarize each segmentation mask accordingly.

The basic training configuration uses a batch size of 4, a learning rate of $10^{-4}$, and the Adam optimizer. Rows 1, 3, 5, and 7 of Tables S.VIII, S.IX, and S.X report the performance under this basic configuration. All reported epoch times and persistent homology computation times are averaged over 480 epochs. In each epoch, persistent homology is computed for 96 distinct $32 \times 32 \times 32$ likelihood maps and the corresponding 96 $32 \times 32 \times 32$ segmentation masks.

The experiments reported in this section were conducted on an Ubuntu 18.04 multi-GPU workstation equipped with NVIDIA GTX 1080 GPUs. TopoGPU was compiled using CUDA 12.1 and g++ 7.5.0, and exposed to Python through a Cython wrapper.

Tables S.VIII–S.X compare CRipser and TopoGPU under different training configurations for 3D U-Net, Attention 3D U-Net, and Residual 3D U-Net. Across all datasets and architectures, TopoGPU consistently reduces the average epoch time. The reduction is mainly due to the substantial decrease in persistent homology (PH) computation time. For example, on PARSE, CRipser requires about 245–272 seconds for PH computation per epoch, whereas TopoGPU reduces this to about 8–9 seconds. Similar behavior is observed on MSD Brain, where TopoGPU reduces PH computation from about 135–155 seconds to about 13–16 seconds.

Changing the learning rate from $10^{-4}$ to $10^{-3}$ has little impact on runtime, which is expected because the learning rate does not change the network architecture, batch size, input volume, or PH workload. In contrast, changing the batch size affects epoch time differently across datasets. On PARSE, increasing the batch size from 4 to 8 generally reduces the average epoch time, due to fewer training iterations per epoch and better GPU utilization. On MSD Brain, however, the epoch time slightly increases, suggesting that the larger batch introduces additional memory pressure or data-processing overhead.

The architectural comparison shows that Residual 3D U-Net generally has the highest runtime, followed by Attention 3D U-Net and standard 3D U-Net, reflecting the additional computations introduced by residual blocks and attention gates. However, the relative advantage of TopoGPU remains consistent across all three architectures. This indicates that the acceleration is primarily determined by the PH backend rather than the segmentation architecture.

Overall, these results demonstrate that TopoGPU significantly improves the efficiency of topology-aware 3D segmentation training. The main performance gain comes from accelerating PH computation, which is the dominant bottleneck when CRipser is used. TopoGPU therefore enables substantially faster training while maintaining comparable network-side GPU memory usage.

TABLE S.VII

ARCHITECTURAL COMPARISON OF THE 3D U-NET, ATTENTION 3D U-NET, AND RESIDUAL 3D U-NET MODELS. HERE, $C_{\text{in}}$ DENOTES THE NUMBER OF INPUT CHANNELS, $C_{\text{out}}$ DENOTES THE NUMBER OF SEGMENTATION CLASSES, $F$ DENOTES `START_FILTERS`, AND $\alpha = 2$ FOR TRILINEAR UPSAMPLING AND $\alpha = 1$ FOR TRANSPOSED-CONVOLUTION UPSAMPLING.

| **Stage** | **3D U-Net** | **Attention 3D U-Net** | **Residual 3D U-Net** |
|---|---|---|---|
| Input | $X \in \mathbb{R}^{C_{\text{in}} \times D \times H \times W}$ | $X \in \mathbb{R}^{C_{\text{in}} \times D \times H \times W}$ | $X \in \mathbb{R}^{C_{\text{in}} \times D \times H \times W}$ |
| Initial block | $\text{DoubleConv3D}(C_{\text{in}}, F)$ | $\text{DoubleConv3D}(C_{\text{in}}, F)$ | $\text{ResidualDoubleConv3D}(C_{\text{in}}, F)$ |
| Encoder level 1 | $\text{MaxPool3D} + \text{DoubleConv3D}(F, 2F)$ | $\text{MaxPool3D} + \text{DoubleConv3D}(F, 2F)$ | $\text{MaxPool3D} + \text{ResidualDoubleConv3D}(F, 2F)$ |
| Encoder level 2 | $\text{MaxPool3D} + \text{DoubleConv3D}(2F, 4F)$ | $\text{MaxPool3D} + \text{DoubleConv3D}(2F, 4F)$ | $\text{MaxPool3D} + \text{ResidualDoubleConv3D}(2F, 4F)$ |
| Encoder level 3 | $\text{MaxPool3D} + \text{DoubleConv3D}(4F, 8F)$ | $\text{MaxPool3D} + \text{DoubleConv3D}(4F, 8F)$ | $\text{MaxPool3D} + \text{ResidualDoubleConv3D}(4F, 8F)$ |
| Bottleneck | $\text{MaxPool3D} + \text{DoubleConv3D}(8F, 16F/\alpha)$ | $\text{MaxPool3D} + \text{DoubleConv3D}(8F, 16F/\alpha)$ | $\text{MaxPool3D} + \text{ResidualDoubleConv3D}(8F, 16F/\alpha)$ |
| Skip connection level 3 | Direct skip: $x_4$ | $\text{AttentionGate3D}(g = x_5, x = x_4)$ | Direct skip: $x_4$ |
| Decoder level 3 | $\text{Up3D}(16F, 8F/\alpha)$ | $\text{Up3D}(16F, 8F/\alpha)$ | $\text{ResidualUp3D}(16F, 8F/\alpha)$ |
| Skip connection level 2 | Direct skip: $x_3$ | $\text{AttentionGate3D}(g = d_3, x = x_3)$ | Direct skip: $x_3$ |
| Decoder level 2 | $\text{Up3D}(8F, 4F/\alpha)$ | $\text{Up3D}(8F, 4F/\alpha)$ | $\text{ResidualUp3D}(8F, 4F/\alpha)$ |
| Skip connection level 1 | Direct skip: $x_2$ | $\text{AttentionGate3D}(g = d_2, x = x_2)$ | Direct skip: $x_2$ |
| Decoder level 1 | $\text{Up3D}(4F, 2F/\alpha)$ | $\text{Up3D}(4F, 2F/\alpha)$ | $\text{ResidualUp3D}(4F, 2F/\alpha)$ |
| Skip connection level 0 | Direct skip: $x_1$ | $\text{AttentionGate3D}(g = d_1, x = x_1)$ | Direct skip: $x_1$ |
| Decoder level 0 | $\text{Up3D}(2F, F)$ | $\text{Up3D}(2F, F)$ | $\text{ResidualUp3D}(2F, F)$ |
| Output layer | $\text{Conv3D}_{1\times1\times1}(F, C_{\text{out}})$ | $\text{Conv3D}_{1\times1\times1}(F, C_{\text{out}})$ | $\text{Conv3D}_{1\times1\times1}(F, C_{\text{out}})$ |
| Output | $\hat{Y} \in \mathbb{R}^{C_{\text{out}} \times D \times H \times W}$ | $\hat{Y} \in \mathbb{R}^{C_{\text{out}} \times D \times H \times W}$ | $\hat{Y} \in \mathbb{R}^{C_{\text{out}} \times D \times H \times W}$ |

TABLE S.VIII

**3D U-NET.** COMPUTATIONAL EFFICIENCY COMPARISON BETWEEN CRIPSER AND TOPOGPU ON THE PARSE AND MSD BRAIN DATASETS UNDER DIFFERENT TRAINING CONFIGURATIONS. ALL EXPERIMENTS USE THE 3D U-NET ARCHITECTURE. AVERAGE EPOCH TIME AND AVERAGE PERSISTENT HOMOLOGY (PH) COMPUTATION TIME ARE REPORTED IN SECONDS, AND PEAK GPU MEMORY IS REPORTED IN MB.

| **Dataset** | **Training Configuration** | | **Avg. Epoch Time** | | **Avg. PH Time** | | **Peak $\text{GPU}_1$ Memory** | | **Peak $\text{GPU}_2$ Memory** | |
|---|---|---|---|---|---|---|---|---|---|---|
| | **Parameter** | **Value** | **CRipser** | **TopoGPU** | **CRipser** | **TopoGPU** | **CRipser** | **TopoGPU** | **CRipser** | **TopoGPU** |
| **PARSE** | Learning rate | $10^{-4}$ | 309.76 | 48.56 | 268.50 | 9.23 | 1912 | 1905 | N/A | 295 |
| | Learning rate | $10^{-3}$ | 313.48 | 47.54 | 270.65 | 9.09 | 1917 | 1905 | N/A | 295 |
| | Batch size | 4 | 309.76 | 48.56 | 268.50 | 9.23 | 1912 | 1905 | N/A | 295 |
| | Batch size | 8 | 284.48 | 43.56 | 245.33 | 8.46 | 2761 | 2761 | N/A | 295 |
| **MSD Brain** | Learning rate | $10^{-4}$ | 180.01 | 60.58 | 135.38 | 13.43 | 6137 | 6137 | N/A | 295 |
| | Learning rate | $10^{-3}$ | 182.24 | 61.72 | 134.97 | 13.18 | 6157 | 6118 | N/A | 295 |
| | Batch size | 4 | 180.01 | 60.58 | 135.38 | 13.43 | 6137 | 6137 | N/A | 295 |
| | Batch size | 8 | 190.93 | 63.27 | 139.56 | 13.45 | 7294 | 7287 | N/A | 295 |

*Note.* PH denotes persistent homology. $\text{GPU}_1$ refers to the device used for network training, whereas $\text{GPU}_2$ refers to the device used by TopoGPU for persistent homology computation.

TABLE S.IX

**ATTENTION 3D U-NET.** COMPUTATIONAL EFFICIENCY COMPARISON BETWEEN CRIPSER AND TOPOGPU ON THE PARSE AND MSD BRAIN DATASETS UNDER DIFFERENT TRAINING CONFIGURATIONS. ALL EXPERIMENTS USE THE ATTENTION 3D U-NET ARCHITECTURE. AVERAGE EPOCH TIME AND AVERAGE PERSISTENT HOMOLOGY (PH) COMPUTATION TIME ARE REPORTED IN SECONDS, AND PEAK GPU MEMORY IS REPORTED IN MB.

| **Dataset** | **Training Configuration** | | **Avg. Epoch Time** | | **Avg. PH Time** | | **Peak $\text{GPU}_1$ Memory** | | **Peak $\text{GPU}_2$ Memory** | |
|---|---|---|---|---|---|---|---|---|---|---|
| | **Parameter** | **Value** | **CRipser** | **TopoGPU** | **CRipser** | **TopoGPU** | **CRipser** | **TopoGPU** | **CRipser** | **TopoGPU** |
| **PARSE** | Learning rate | $10^{-4}$ | 315.24 | 52.11 | 270.11 | 8.81 | 2221 | 2221 | N/A | 295 |
| | Learning rate | $10^{-3}$ | 310.87 | 53.22 | 266.98 | 8.99 | 2221 | 2263 | N/A | 295 |
| | Batch size | 4 | 315.24 | 52.11 | 270.11 | 8.81 | 2221 | 2221 | N/A | 295 |
| | Batch size | 8 | 292.39 | 43.68 | 254.62 | 7.97 | 2893 | 2897 | N/A | 295 |
| **MSD Brain** | Learning rate | $10^{-4}$ | 181.80 | 58.40 | 135.93 | 15.74 | 7543 | 7543 | N/A | 295 |
| | Learning rate | $10^{-3}$ | 179.93 | 58.89 | 139.29 | 15.53 | 7544 | 7474 | N/A | 295 |
| | Batch size | 4 | 181.80 | 58.40 | 135.93 | 15.74 | 7543 | 7543 | N/A | 295 |
| | Batch size | 8 | 198.03 | 65.92 | 143.44 | 15.93 | 8128 | 8128 | N/A | 295 |

*Note.* PH denotes persistent homology. $\text{GPU}_1$ refers to the device used for network training, whereas $\text{GPU}_2$ refers to the device used by TopoGPU for persistent homology computation.

TABLE S.X

**RESIDUAL 3D U-NET.** COMPUTATIONAL EFFICIENCY COMPARISON BETWEEN CRIPSER AND TOPOGPU ON THE PARSE AND MSD BRAIN DATASETS UNDER DIFFERENT TRAINING CONFIGURATIONS. ALL EXPERIMENTS USE THE RESIDUAL 3D U-NET ARCHITECTURE. AVERAGE EPOCH TIME AND AVERAGE PERSISTENT HOMOLOGY (PH) COMPUTATION TIME ARE REPORTED IN SECONDS, AND PEAK GPU MEMORY IS REPORTED IN MB.

| **Dataset** | **Training Configuration** | | **Avg. Epoch Time** | | **Avg. PH Time** | | **Peak $GPU_1$ Memory** | | **Peak $GPU_2$ Memory** | |
|---|---|---|---|---|---|---|---|---|---|---|
| | **Parameter** | **Value** | **CRipser** | **TopoGPU** | **CRipser** | **TopoGPU** | **CRipser** | **TopoGPU** | **CRipser** | **TopoGPU** |
| **PARSE** | Learning rate | $10^{-4}$ | 329.82 | 64.28 | 271.62 | 9.06 | 2283 | 2277 | N/A | 295 |
| | Learning rate | $10^{-3}$ | 322.76 | 64.84 | 269.91 | 9.31 | 2223 | 2223 | N/A | 295 |
| | Batch size | 4 | 329.82 | 64.28 | 271.62 | 9.06 | 2283 | 2277 | N/A | 295 |
| | Batch size | 8 | 301.84 | 56.61 | 249.13 | 8.32 | 3295 | 3295 | N/A | 295 |
| **MSD Brain** | Learning rate | $10^{-4}$ | 202.28 | 71.59 | 143.84 | 14.80 | 5853 | 5993 | N/A | 295 |
| | Learning rate | $10^{-3}$ | 206.53 | 70.59 | 144.99 | 14.42 | 5789 | 5993 | N/A | 295 |
| | Batch size | 4 | 202.28 | 71.59 | 143.84 | 14.80 | 5853 | 5993 | N/A | 295 |
| | Batch size | 8 | 213.50 | 73.12 | 154.58 | 15.56 | 7207 | 7207 | N/A | 295 |

*Note.* PH denotes persistent homology. $GPU_1$ refers to the device used for network training, whereas $GPU_2$ refers to the device used by TopoGPU for persistent homology computation.



\end{document}